\documentclass{article}

\usepackage{hyperref}

\usepackage[preprint]{icml2023}

\usepackage{graphicx} %
\usepackage{natbib}

\title{A Strong Baseline for Batch Imitation Learning}

\usepackage{microtype}
\usepackage{subcaption}
\usepackage{booktabs} %
\usepackage{xcolor}
\usepackage{mathtools}
\usepackage{amsthm}
\usepackage{amssymb}
\usepackage[capitalize,noabbrev]{cleveref}
\hypersetup{
    colorlinks=true,
    linkcolor=red,
    filecolor=magenta,
    urlcolor=blue,
    citecolor=purple,
    pdftitle={Overleaf Example},
    pdfpagemode=FullScreen,
    }

\newtheorem{theorem}{Theorem}

\newtheorem{lemma}{Lemma}
\newtheorem*{exlemma}{Lemma}
\newtheorem{claim}{Claim}

\theoremstyle{remark}
\newtheorem*{remark}{Remark}

\newcommand{\algolong}{Imitation Learning by Batch Reinforcement Learning}
\newcommand{\algo}{ILBRL}

\DeclareMathOperator*{\argmax}{arg\,max}
\DeclareMathOperator*{\argmin}{arg\,min}

\begin{document}
\twocolumn[
\icmltitle{A Strong Baseline for Batch Imitation Learning}

\begin{icmlauthorlist}
\icmlauthor{Matthew Smith}{oxford,spotifyint}
\icmlauthor{Lucas Maystre}{spotify}
\icmlauthor{Zhenwen Dai}{spotify}
\icmlauthor{Kamil Ciosek}{spotify}
\end{icmlauthorlist}

\icmlaffiliation{oxford}{Department of Computer Science, University of Oxford}
\icmlaffiliation{spotifyint}{Research completed while on internship at Spotify.}
\icmlaffiliation{spotify}{Spotify}

\icmlcorrespondingauthor{Matthew Smith}{msmith@cs.ox.ac.uk}
\icmlcorrespondingauthor{Kamil Ciosek}{kamil.ciosek@posteo.net}

\icmlkeywords{Machine Learning, Reinforcement Learning, Imitation Learning, Offline RL, Batch RL, ICML}

\vskip 0.3in
]

\printAffiliationsAndNotice{}  %

\begin{abstract}
Imitation of expert behaviour is a highly desirable and safe approach to the problem of sequential decision making. We provide an easy-to-implement, novel algorithm for imitation learning under a strict data paradigm, in which the agent must learn solely from data collected a priori. This paradigm allows our algorithm to be used for environments in which safety or cost are of critical concern. Our algorithm requires no additional hyper-parameter tuning beyond any standard batch reinforcement learning (RL) algorithm, making it an ideal baseline for such data-strict regimes. Furthermore, we provide formal sample complexity guarantees for the algorithm in finite Markov Decision Problems. In doing so, we formally demonstrate an unproven claim from \citet{kearns1998finite}. On the empirical side, our contribution is twofold. First, we develop a practical, robust and principled evaluation protocol for offline RL methods, making use of only the dataset provided for model selection. This stands in contrast to the vast majority of previous works in offline RL, which tune hyperparameters on the evaluation environment, limiting the practical applicability when deployed in new, cost-critical environments. As such, we establish precedent for the development and fair evaluation of offline RL algorithms. Second, we evaluate our own algorithm on challenging continuous control benchmarks, demonstrating its practical applicability and competitiveness with state-of-the-art performance, despite being a simpler algorithm.
\end{abstract} 

\section{Introduction}
Optimal sequential decision making remains a  problem which is challenging but crucial for tackling many practical applications. While reinforcement learning (RL) algorithms are theoretically well suited to this task, in practice they are limited by safety concerns \citep{russell2015research}, unknown data requirements, and high sensitivity to problem design parameters \citep{russell2010artificial}. Imitation learning (IL) presents a flexible paradigm that sidesteps these concerns, but requires expert demonstrations.

Ideally, in order to be used in the widest variety of practical settings, we consider two key desiderata for (offline) IL algorithms: First, learning should occur only on data gathered a priori from known policies --- no additional interaction with the true environment should be needed. This ensures that the IL algorithm can be used in expensive or dangerous environments. Second, the algorithm should not require access to the true reward, which is unknown in general. This allows the agent to learn exclusively from a dataset, without need for a handcrafted reward, which may be expensive to evaluate and lead to unpredictable behaviour.

Perhaps surprisingly, to our knowledge, no known IL algorithm thus far has demonstrated both of these properties, in part due to the manner in which prior algorithms make use of environment interaction to tune hyperparameters. In this work, we approach these issues by developing a simple IL algorithm and evaluation protocol which can be used to develop, compare, and tune offline RL algorithms.

Our algorithm is inspired by support matching approaches to IL, which look to simplify the IL problem by constructing a reward function offline which assigns high reward to policies that match the support of the expert policy in the dataset \citep{wang2019random,reddy2019sqil,ciosek2022imitation}. Then a policy is learned by a single call to an RL solver with this fixed reward function. However, thus far, these algorithms require online interaction with the environment in order to optimise for the constructed reward.

Other recent works in offline IL \citep{zolna2020offline, ma2022smodice} have focused on eliminating the need for environmental interaction in order to imitate expert policies. These works introduce, alongside a small expert dataset, a second, much larger dataset which is used to learn dynamics of the Markov Decision Process (MDP) in areas potentially not covered by the expert agent. We will make use of this approach in our algorithm as well. These prior works are limited by their complexity: they involve multi-stage learning procedures involving density estimation as well as off-policy RL methods. Furthermore, the evaluation of these algorithms makes use of hyperparameter tuning on the environment, violating our desiderata above.

\paragraph{Contributions}
In order to make progress in offline IL, we develop a simple, easily optimised support matching algorithm, with theoretical guarantees in the tabular RL setting.

Building on the work of \citet{ciosek2022imitation}, we provide formal bounds on the performance of a policy learned via our algorithm for discrete MDPs. These bounds depend on the size of the MDP, the size of the expert dataset, and the size and quality of the exploratory data-generating policy. As a corollary, we prove a conjecture by \citet{kearns1998finite}, which to our knowledge has never been demonstrated formally until now.

Equipped with these bounds, we go on to empirically demonstrate the effectiveness of a relaxation of our algorithm applied in challenging continuous control environments. In order to compare offline RL methods fairly, with explicit data requirements, we develop a novel and pragmatic hyperparameter tuning protocol for offline RL methods, which requires minimal knowledge of the true environment. We show that our simple algorithm is competitive with state-of-the-art IL methods in a challenging data regime in which Behavioural Cloning (BC) fails. These results are statistically robust, making use of the evaluation methods described by \citet{agarwal2021deep}, which establishes a precedent for more robust benchmarking.

Overall, our study finds the proposed algorithm to be both simple and effective, making it a pragmatic baseline for future work on IL in challenging, low expert data regimes.

\section{Preliminaries}

In this section, we provide a brief technical overview of the formalism employed in the rest of the paper and establish the notation that we will apply.

\paragraph{Markov Decision Processes \& Reinforcement Learning}

Markov Decision Processes (MDPs) provide a formalism for agents in sequential reasoning task. An \textit{average reward} MDP is given by the tuple: $\left\langle \mathcal{S}, \mathcal{A}, P, R, P_0 \right\rangle$, where $\mathcal{S}$ is the state space, $\mathcal{A}$ the action space, $P:\mathcal{S} \times \mathcal{A} \to U(\mathcal{S})$ the transition distribution, where $U(\mathcal{S})$ is the set of probability distributions over states, $R:\mathcal{S} \times \mathcal{A} \to [0, 1]$ a bounded deterministic reward function, and $P_0$ the distribution over initial states. We also denote $S = |\mathcal{S}|$, $A = |\mathcal{A}|$.
In this work, we consider agents which act according to a deterministic policy, $\pi: \mathcal{S} \to \mathcal{A}$, which maps states to the action taken in that state. We also assume that such policies induce \textit{ergodic} Markov chains for simplicity. In an average reward MDP, we look to find a policy which optimises the average per-step expected reward, given by
\[
    \mu_{\pi} = \lim_{N\to\infty}\frac{1}{N}\mathbb{E}\left[
        \sum_{i=0}^N r\left(s_i, \pi(s_i)\right)
    \right],
\]
where expectation is taken over draws from $P_0$ and $P\left(s_i, \pi(s_i)\right)$.

A \textit{discounted reward} MDP adds, in addition to the tuple described above, a \textit{discount factor}, $\gamma$, which determines an effective time horizon for the MDP. In a discounted reward MDP, values are state dependent, and given by
\[
    V_\pi(s) = \mathbb{E}\left[\sum_{i=0}^{\infty} \gamma^i r\left(s_i, \pi(s_i)\right) \Big\vert s_0 = s\right].
\]
where the conditional expectation is again taken over draws from the MDP dynamics. Related to the value function is the \textit{action} value function, which determines the value of taking an action in a state, then proceeding to follow $\pi$: $Q_\pi(s,a) := r(s,a) + \gamma\mathbb{E}_{s' \sim P(s,a)}\left[V_{\pi}(s')\right]$.

Of special importance to us is the fact that average rewards and the average values only differ by a constant factor of $(1 - \gamma)$ --- see e.g. the work of \citet{kakade2001optimizing}:
\[\mu^{\pi}  = (1 - \gamma)\sum_s \rho_{\pi}V^{\pi}(s),\]
where $\rho^{\pi}$ is the steady-state distribution of the Markov chain induced by policy $\pi$ on the MDP.

\paragraph{Imitation Learning}
Imitation learning looks to recast the problem of learning an optimal policy by making use of a fixed dataset of expert behaviour, $D_E$, and by learning a policy which achieves average reward close to the one used to produce the dataset. Formally, for an expert policy, $\pi_E$ and any reward function $f$ with range $[0,1]$, we look to learn $\pi$ such that
\begin{equation}
    \label{eqn:imlearning}
    \mathbb{E}_{\rho_{\pi_{E}}}\left[f(s,a)\right] - \mathbb{E}_{\rho_{\pi}}\left[f(s,a)\right] \leq \epsilon.
\end{equation}

In addition to the finite expert dataset, we consider an additional finite \textit{exploratory} dataset, $D_X$ which consists of transitions generated by an arbitrary policy, $\pi_X$. The policy must cover the MDP, meaning it must be possible for it to visit any given state and action of the MDP in a finite number of steps. This is to say, the steady state probability obeys $\rho_{\pi}(s, a) \geq p_{\min} > 0$. The purpose of this dataset is to provide the offline agent with information about the MDP dynamics outside of the states covered by the expert policy. Training thus occurs on the unified dataset $D_U = D_E \cup D_X$.

\paragraph{Total Variation Distance}
The total variation distance (TV) formalises a notion of how far apart two probability distributions are, in terms of how they assign different probabilities to the same events. Formally, for discrete distributions $\rho_1$ and $\rho_2$, over an event set $\Omega$, we have
\[
\left \lVert \rho_1 - \rho_2 \right \rVert_{TV} := 
    \sup_{M \subseteq \Omega}\left\lvert \rho_1(M) - \rho_2(M) \right \rvert.
\]
It is well-known that bounding the total variation distance between the expert policy and a policy learned via imitation learning also provides a bound on \cref{eqn:imlearning} --- see, e.g. \citet{ciosek2022imitation}. As such our bound will also look to bound the TV between the expert distribution and the policy learned by our algorithm.
The total variation distance is also used to define the \emph{mixing time} $t^\pi$ of an ergodic policy $\pi$, as the smallest $t$ such that
\[\max_{s} \lVert \mathbf{1}(s')^{\top}P_{\pi}^t - \rho_\pi\rVert_{TV} \leq \frac{1}{4},\]
where $P_{\pi}$ is the transition matrix for the Markov chain induced by $\pi$ and transition function $P$. This will be used in our analysis to represent the quality of the exploratory policy, and the coverage of the expert data over the state space.

\section{Imitation Learning by Batch RL}
In order to introduce our algorithm, \algolong\ (\algo), we define the \textit{intrinsic} reward, which, for discrete MDPs, is given by  
\[
\hat{r}(s,a) = \mathbf{1}_{D_E}((s,a)),
\]
where $\mathbf{1}_{D_E}$ is the indicator function, which maps to $1$ if a state action pair is in the expert dataset, and 0 otherwise. This matches exactly the function used by \citet{wang2019random} and \citet{ciosek2022imitation}. However, we implement a subtle but important shift. Instead of making a call to an arbitrary RL solver in order to optimise returns under this intrinsic reward through unlimited interaction with the environment, we make a call to an \textit{offline} RL algorithm, which in general may not be able to perfectly optimise for our intrinsic reward. This significantly complicates analysis, as we need to manage the inefficiency and statistical dependence on finite data that comes from offline RL algorithms, alongside the fact that optimising for the intrinsic reward may not lead to a good policy if there is not enough expert data. Pseudocode for the overall process is given in Algorithm 1.

In order to simplify analysis, we make use of an off-policy RL algorithm, developed by \citet{kearns1998finite}, known as \textit{phased Q-Learning}. The algorithm is somewhat impractical, but its analysis is tractable. We develop favourable asymptotical guarantees on its sample complexity in subsequent sections.

\begin{algorithm}
\caption{Imitation Learning by Batch RL \label{algo:ilbrl}}
\begin{algorithmic}[1]
    \renewcommand{\algorithmicrequire}{\textbf{Input:}}
    \renewcommand{\algorithmicensure}{\textbf{Output:}}
    \REQUIRE Expert dataset: $D_E$, Exploration dataset: $D_X$
    \ENSURE Imitation policy: $\pi$
    \STATE $D_U \gets D_E \cup D_X$
    \FOR{$s, a \in D_U$}
        \STATE $\hat{r}(s,a) \gets \mathbf{1}_{D_E}((s,a))$
    \ENDFOR
    \STATE  $\pi \gets$ BatchRL($D_U, \hat{r}$) 
    \RETURN $\pi$
\end{algorithmic}
\end{algorithm}

\section{Prior Work}
\paragraph{Behavioural Cloning}
There exists a wide body of prior work which looks to solve the imitation learning problem. The simplest and oldest form is known as behavioural cloning, which simply optimises policy parameters in order to match the expected state-conditional probability of taking the actions in the expert dataset. However, this formulation fails to accommodate for the sequential nature of the MDP, leading to distribution shift \citep{ross10efficient}, and severely increasing the amount of expert data needed to learn a good policy.

\paragraph{Inverse RL}
Inverse RL (IRL) proposes to evade this issue by learning a reward function which the expert is (implicitly) maximising --- as indicated by the data --- then optimising a policy to solve the original MDP augmented with this reward function. Early formulations of this problem required full solutions to candidate MDPs at each iteration \citep{ng2000algorithms, ziebart2008maximum}, while more recent \textit{adversarial} IRL methods \citep{ho2016generative} use GAN-style training to construct candidate reward functions and policies online in order to minimise a divergence between the learned policy and the expert data distribution. However, these methods are known to be difficult to train, and require continuous interaction with the environment, which is not available in our setting, which constrains interaction with the MDP to a finite dataset.

\paragraph{Support Matching}
In order to avoid the instabilities of adversarial IRL methods, and to limit the number of calls to an RL solver needed, \citet{wang2019random} propose a phased approach, which looks to construct a reward function such that state-actions visited by (or close to those visited by) a deterministic expert are rewarding, while state-actions far from the dataset are not. A reward function is first estimated using only the expert data, either by direct labelling, or an approximate support estimation method. Then, a single call to an RL solver is made, using the rewards supported by the expert. \citet{ciosek2022imitation} develops theory for support matching methods, providing formal guarantees for reward labelling on discrete MDPs, as well as a heuristic relaxation of the labelling which has been shown to work in continuous environments. Our work builds on these support matching results, but looks to do so in a strictly batch setting, without the need to call an online RL solver.

\paragraph{Offline Imitation Learning}
Several previous methods employ the setting assumed by our work, looking to learn to mimic an expert policy strictly from data. \citet{zolna2020offline} propose an adaptation of adversarial IRL methods for batch settings, though this approach still suffers from training challenges and nonstationarity. More recently, \citet{kim2021demodice} establish the use of a hybrid dataset composed of expert and exploratory data, making use of a marginalised density ratio method \citep{liu2018breaking, nachum2019dualdice} in order to minimise the KL-divergence between the learned policy and expert behaviour. Concurrently, \citet{ma2022smodice} discuss a paradigm which leverages a more general class of $f$-divergences in a similar way.

\section{Theory of \algo}
\label{sec:theory}

\subsection{Theory of Phased Q-Learning}
Phased Q-learning can be seen as an approximate form of value iteration. In this algorithm, rather than sampling states and actions individually, each action at each state across the entire environment is sampled simultaneously $m$ times. This data generation process is known as \textit{parallel sampling}. Parallel sampling can be simulated using standard RL sampling based on rollouts of a policy in the MDP, which we characterise theoretically later in this section. It allows us to update the estimated state-action value at each state and action simultaneously according to the optimal Bellman equation using the deterministic reward and \textit{empirical bootstrap}:
\[
    Q_{i+1}(s, a) 
        = r(s,a) + \gamma \frac{1}{m}\sum_{k=1}^m \max_{a'} Q_{i}(s'_k, a').
\]
Since the empirical bootstrap will concentrate around the true value, we can see that, with enough data, this becomes equivalent to performing value iteration. Here we prove that this is an efficient offline RL algorithm. First, we need to demonstrate that the amount of data needed to achieve bounded error at the end of the evaluation process is polynomial in the appropriate quantities, which was accomplished in the original work by \citet{kearns1998finite}. Second, we need to show that we can simulate sampling from the parallel model efficiently. \citet{kearns1998finite} postulate that such a procedure should be possible under assumptions concerning the quality of the exploration policy, but a formal proof has never been provided until now.

We use the notation from \citet{kearns1998finite}. $m$ is the number of samples from the parallel sampler per state-action pair, $\ell$ is the number of iterations of phased Q-learning to perform, $s$ and $a$ are states and actions, and $P_{ss'}^a$ is the transition distribution, $r$ is a deterministic reward function, $\hat{V}_i(\cdot), \hat{Q}_i(\cdot,\cdot)$ are our estimated value functions at iteration $i$, $\bar{V}_i(\cdot), \bar{Q}_i(\cdot,\cdot)$ are the true value functions of the policy greedy with respect to $\hat{Q}_i$, $V_i(\cdot), Q_i(\cdot,\cdot)$ is the output of $i$ rounds of policy iteration, $V^*(\cdot), Q^*(\cdot,\cdot)$ are the optimal value functions.

We start by assuming that the empirical bootstrap of our estimated function is $\eta'$-concentrated around the bootstrap on the true transition distribution. We will later take $m$ large enough for this to hold.
For all $s, a$, and $i\leq\ell$, we have
\begin{align}
    \left\lvert 
        \frac{1}{m}\sum_{k=1}^m \hat{V}_i(s'_k) 
        - \sum_{s'}P_{ss'}^a \hat{V}_i(s')
    \right\rvert
    \leq \eta',
    \label{eqn:val_concentration}
\end{align}
where $s'_k$ is the $k$th sampled subsequent state from the parallel sampler at $s,a$.

With this assumption, we are able to formalise the notion that performing phased Q-learning is very close to performing true value iteration, as suggested by the following lemma.

\begin{lemma}
    \label{lemma:value_gap}
Given \cref{eqn:val_concentration}, the difference between the value output by phased Q-learning after $\ell$ steps and the optimal value is given by
\begin{align*}
\zeta(s,a)
    := \left\lvert \hat{Q}_{\ell}(s, a)  - Q^*(s,a) \right\rvert
    \leq \frac{\eta' \gamma + \gamma^{\ell}}{1 - \gamma}
\qquad \forall s, a.
\end{align*}
\end{lemma}
The proof is left to the appendix, which makes use a recurrence relation between updates of phased Q-learning.

We now look to bound the overall suboptimality of the policy output by phased Q-learning. We define the discounted regret of the policy at step $i$ on the discounted MDP as the expectation of the difference between the discounted return achieved by the learned policy and that of the optimal policy, with expectation taken over the initial state distribution:
\[
    R_\gamma(i) := \mathbb{E}_{s_0}\left[V^*(s_0) - \bar{V}_i(s_0) \right].
\]
Equipped with the pointwise value gap, we can bound the overall regret with the following lemma.
\begin{lemma}
    \label{lemma:regret}
    Given \cref{eqn:val_concentration}, the policy output by phased Q-learning after $\ell$ steps achieves regret of at most
    \[
        R_\gamma(\ell) \leq 
            \frac{2}{(1-\gamma)}
            \left(
                \frac{\eta' \gamma  + \gamma^\ell}{1 - \gamma}
            \right).
    \]
\end{lemma}

We provide a complete proof of \cref{lemma:regret} in the appendix.
Since we now have a bound on the overall regret of our algorithm, bounding by $\epsilon$, our overall error margin, allows us to solve for the relevant quantities, eventually leading to a lower bound on the number of samples needed to achieve such error.

Once we have solved for the minimum acceptable concentration error $\eta'$ in terms of our overall allowable discounted regret, $\eta$, we can solve for $m$, the number of samples we need for \cref{eqn:val_concentration} to hold with high probability. This is accomplished in the following lemma.

\begin{lemma}
\label{lemma:phased_q}
Taking $m \ell$ samples from the parallel model for phased Q-learning, such that
\begin{align*}
m \ell
    \geq \log\left(\frac{2 \ell SA}{\delta'}\right)\frac{2 \gamma^2 \ell}{(1-\gamma)^2(\eta(1 - \gamma)^2 - 2\gamma^\ell)^2},
\end{align*}
we obtain discounted regret less than $\eta$ with probability at least $1 - \delta'$.
\end{lemma}
The proof is left to the Appendix. It involves bounding by the overall error $\eta$ and solving for $m\ell$ using the Hoeffding inequality.

Now we only need to concern ourselves with the gathering of these $m \ell$ samples from the parallel model. We do so by appealing to a classic result applying to ergodic Markov chains, which suggests that sampling the current state and action after following the chain for several steps is close to sampling from the stationary distribution, allowing for gathering of independent samples.

\begin{lemma}
\label{lemma:lucas_lemma}
The number of sample transitions from the exploration policy needed to perfectly simulate the parallel sampling model with failure probability less than $\delta''$ is upper bounded by
\begin{align*}
\frac{2 t^{\pi_X}}{\log(2) p_{\rm min}}  \log \left( \frac{2}{p_{\rm min}} \right)  \log \left( \frac{SA}{\delta''} \right).
\end{align*}
\end{lemma}
The proof can be found in the Appendix. \cref{lemma:lucas_lemma} combined with \cref{lemma:phased_q} give the overall sample complexity of the phased-Q learning stage of our algorithm.

\subsection{Bounding the TV Between Expert and Imitator}

With a good sense for the data requirements of our offline RL algorithm, we can proceed to ensure performance of our algorithm. We do so by invoking a key lemma of \citet{ciosek2022imitation}. However, this lemma applies only to the average reward setting, while our results thus far concern policies optimised for discounted reward. As such we must first bound the suboptimality of the learned policy with respect to the average value attained. We achieve this with a relaxation of Theorem 1 by \citet{kakade2001optimizing}.

\begin{lemma}
    \label{lemma:discounted_to_average}
    Suppose policy $\pi$ is $\epsilon$-suboptimal or better at each state with respect to the discounted return. Let $\pi^*$ be the optimal policy under the average reward criterion. We assume  $P_{\pi^*}$ has $S$ distinct eigenvalues. Taking $\Sigma$ to be the matrix of right eigenvectors of $P_{\pi^*}$ with corresponding eigenvalues: $\lambda_1  = 1 > ... \geq \vert \lambda_S \vert$, we have
    \[
        \mu^{\pi} \geq \mu^{\pi^*} - \kappa(\Sigma)\lVert \mathbf{r} \rVert \frac{1 - \gamma}{1 - \gamma \lvert \lambda_2 \rvert} - (1 - \gamma) \epsilon,
    \]
    where $\kappa(\Sigma) := \lVert \Sigma \rVert_2 \lVert \Sigma^{-1} \rVert_2$ is the condition number of the matrix $\Sigma$, and $\mathbf{r}$ is the vector with entries $r(s, \pi^{*}(s))$ for each state.
\end{lemma}
The proof is in the appendix. It follows closely that of \citet{kakade2001optimizing}.

Lemma \ref{lemma:discounted_to_average} gives us a lower bound on the suboptimality with respect to the average reward of a policy $\epsilon$-suboptimal with respect to the discounted reward. Next, we need to ensure that it is possible to achieve a policy which performs well with respect to the intrinsic reward. This is accomplished by invoking Lemma 5 of \citet{ciosek2022imitation}, reproduced here for convenience.
 
\begin{lemma}[Lemma 5, \citealp{ciosek2022imitation}]
\label{lemma:ciosek_5}
 Given an expert dataset consisting of $\lvert D_E \rvert$ points, 
 a policy $\pi$ which maximises the average intrinsic reward achieves an expected average intrinsic reward of at least
 \[ \mu_{\textrm{int}}^\pi := 1 - \nu - \sqrt{\frac{8St^{\pi_E}}{\lvert D_E \rvert}},\]
 for all error terms $\nu > 0$, with probability at least
 \[1 - \delta''' := 1 - 2 \exp\left(-\frac{\nu^2 \lvert D_E \rvert}{4.5 t^{\pi_E}}\right).\]
 
\end{lemma}

To connect high performance on the intrinsic reward with the performance on the extrinsic reward, we again invoke a key lemma of \citet{ciosek2022imitation},
\begin{lemma}[Lemma 7, \citealp{ciosek2022imitation}]
\label{lemma:ciosek_7}
An agent which achieves an average reward of $1 - \epsilon'$ on the intrinsic reward MDP also achieves average reward of
\[(1 - \epsilon')\mu^{\pi_E} - 4 t^{\pi_E} \epsilon'\]
on the true MDP.
\end{lemma}

We are now ready to conclude, evaluating the overall sample complexity of our algorithm.

\begin{theorem}
\label{thm:regret_main}
Consider an instance of \cref{algo:ilbrl}. In our setting (Assumptions \ref{assumption:deterministic}-\ref{assumption:kakade} in the appendix),
in order to achieve regret on the average reward problem of less than $\epsilon$, or total variation distance between the expert and the imitation policy less than $\epsilon$, with probability $1 - \delta$, we need to sample 
\begin{gather*}
\max\left\{
    \frac{1}{\epsilon^2}32St^{\pi_E}(1+4t^{\pi_E})^2, 
    \right. \quad \quad \quad \quad \quad \quad \quad \quad \\ 
    \quad \quad \quad \quad \quad \quad \quad \quad \left. 
    \frac{
        4.5t^{\pi_E}16(1 +  4 t^{\pi_E})^2 \log(4 / \delta)
    }{\epsilon^2}
\right\}
\end{gather*}
transitions from the expert policy, and 
\begin{gather*}
\mathcal{O}\left(
    \frac{t^{\pi_X}}{p_{\rm min}}
    \log \left( \frac{1}{p_{\rm min}} \right) 
    \log^2 \left( \frac{SA}{\delta} \right)
        \right. \quad \quad \quad \quad \quad \quad \\
        \quad \quad \quad \quad \quad \quad \left.
    \log\left(
        \frac{t^{\pi_E}\beta}{\epsilon(1 - \lambda)}
    \right)
    \frac{(t^{\pi_E})^8\beta^6}{\epsilon^8 (1 - \lambda)^6}
\right)
\end{gather*}
from the exploratory policy, where $\beta =  \kappa(\Sigma)\lVert \mathbf{r} \rVert$, as described in \cref{lemma:discounted_to_average}.
\end{theorem}
We defer the proof to the Appendix, which largely consists of disentangling the relevant constants. We note that the bound on the amount of expert data needed differs from that of \citet{ciosek2022imitation} only in terms of a constant factor. This indicates that the burden of offline learning falls mostly on the exploratory policy. We use big-O notation to denote the number of exploratory samples needed in order to provide intuition, but specific constants can be found in the proof of \cref{thm:regret_main}.

\section{Empirical Evaluation}
In order to evaluate the efficacy of \algo, we evaluate it using a relaxation of our density labelling for the continuous control setting, similar to that of \citet{ciosek2022imitation}.

\subsection{Realistic Evaluation of Offline RL Methods}

\begin{table*}[ht]
    \centering \small
\begin{tabular}{|l|c|c|c|c|}
    \hline
    \textbf{Algorithm} & \textbf{Hopper}  & \textbf{Halfcheetah} & \textbf{Walker2d} & \textbf{Ant}\\
    \hline\hline
    ILBRL & $\mathbf{64.634}\pm8.417$ & $\mathbf{49.856}\pm1.229$ & $\mathbf{80.681}\pm6.154$ & $\mathbf{104.907}\pm1.109$\\
    BC & $\mathbf{69.459}\pm15.005$ & $4.884\pm0.338$ & $27.04\pm5.937$ & $21.666\pm5.17$\\
    SMODICE & $53.228\pm2.508$ & $42.662\pm0.422$ & $\mathbf{75.942}\pm0.794$ & $90.05\pm1.436$\\
    ORIL & $5.7\pm5.214$ & $\mathbf{55.079}\pm4.207$ & $25.268\pm14.239$ & $35.81\pm2.296$\\
    \hline
    (Exploration) & $55.62$ & $42.686$ & $67.371$ & $88.333$\\
    (Expert) & $111.016$ & $94.371$ & $107.747$ & $121.639$\\
    \hline
\end{tabular}
\caption{Mean performances on environment during evaluation for three training seeds, averaged over 30 rollouts. Error measures are given by a 95\% confidence interval based on the mean performance per training seed. }
\label{tab:mean_table}
\end{table*}

While batch RL algorithms purport to circumvent dangerous exploration problems and data complexity issues by using a dataset of fixed size, it is standard for batch methods to perform hyperparameter validation on the true environment, with access to conceivably unlimited data and the true reward function. This means that the data requirements of batch RL algorithms are not made explicit, and IL algorithms are affected by a data leak, effectively meaning they are tuned to optimise the true rewards. In an extreme case we can imagine that all parameters are taken to be hyperparameters, and we essentially train a ``batch'' algorithm, though all optimisation happens on data from the true environment.

This issue has been noted recently by \citet{paine2020hyperparameter}, who suggest using an independent offline policy evaluation (OPE) procedure on the policy output by the offline RL algorithm. Hyperparameters are optimised in accordance with the estimated values of the initial states in the dataset. However, little attention is given to the fact that the OPE is itself a challenging and provably difficult problem \citep{wang2020statistical}, requiring hyperparameter tuning of its own. In order to address these issues, we develop a novel evaluation protocol which enables tuning of OPE in a practically reasonable and safe manner. The details of our approach are below, with pseudocode in the appendix.

\paragraph{Data Partitioning}
In order to prevent overfitting to our limited dataset and biasing our performance estimates, as in supervised learning, we first partition our dataset into a training set and evaluation set. The policy evaluation set is then further partitioned into an \textit{evaluation training set}, and an \textit{evaluation validation set}, preventing overfitting of the OPE algorithm to the evaluation data. 

\paragraph{Tuning OPE}
In order to tune the OPE algorithm, at least two deterministic policies of different known quality are needed. In addition, we require that the dataset be composed of transitions which include the actions taken at the subsequent state. In this work we employ Expected SARSA \citep{van2009theoretical}, though in principle, other methods for OPE can be used. ESARSA is then tuned on the data with the known, ``safe'' policies to ensure convergence to the values of the known policies, as well as good separation of policy values. Since data is generated by multiple policies, one key hyperparameter of the policy evaluation process is the composition of the evaluation dataset --- we need to find a reasonable data distribution, such that state distribution induced by our target policies are not too far from the distribution of the dataset, a key factor which leads to divergence. These hyperparameters are then used to evaluate the policies learned via offline RL. While in general this may not lead to a convergent OPE algorithm for all possible learned policies, such an algorithm is not yet possible in general. Instead, our procedure provides a consistent evaluation protocol without the use of new data from the environment, enabling principled selection of policies and grounding of the learned values in the environment at hand.

\paragraph{Policy Selection}
Each batch RL method is trained across several hyperparameter configurations, and several seeds. See Appendix 2 for methodological details. Using the evaluation hyperparameters tuned on the known safe policies, evaluation of a single hyperparameter and random seed setting for the batch algorithm is then performed across several evaluation seeds in order to estimate the true value of the learned policy. Policy evaluation is run for a fixed number of steps, determined by the amount of time it took for evaluation to converge on the known policies. Performance of a control hyperparameter setting is then given by the mean output of the learned value function on a held-out dataset consisting of initial states. The policy which achieves the highest mean convergent value across seeds and rollouts is then chosen for evaluation on the true environment, which forms the results that are reported in the next section (details of this procedure are provided in the appendix).

\paragraph{Imitation Learning}
Our paper is concerned with IL, which in its most general form completely lacks access to the true reward function. However, while our algorithm needs no access to true rewards, in order to implement our protocol for the purpose of evaluation, knowledge of true rewards is needed for the transitions in the dataset. These rewards are used for policy evaluation only: importantly they are not needed for \algo. To our knowledge, no previous work in IL is able to strictly adhere to an unknown reward function, as hyperparameters are optimised according to returns on the true environment. We see the present work as a solution to this problem by limiting reward function access to an algorithm-independent policy evaluation procedure and a restricted dataset, rather than the complete MDP.

\begin{figure}[ht]
    \centering
    \includegraphics[width=0.7\linewidth]{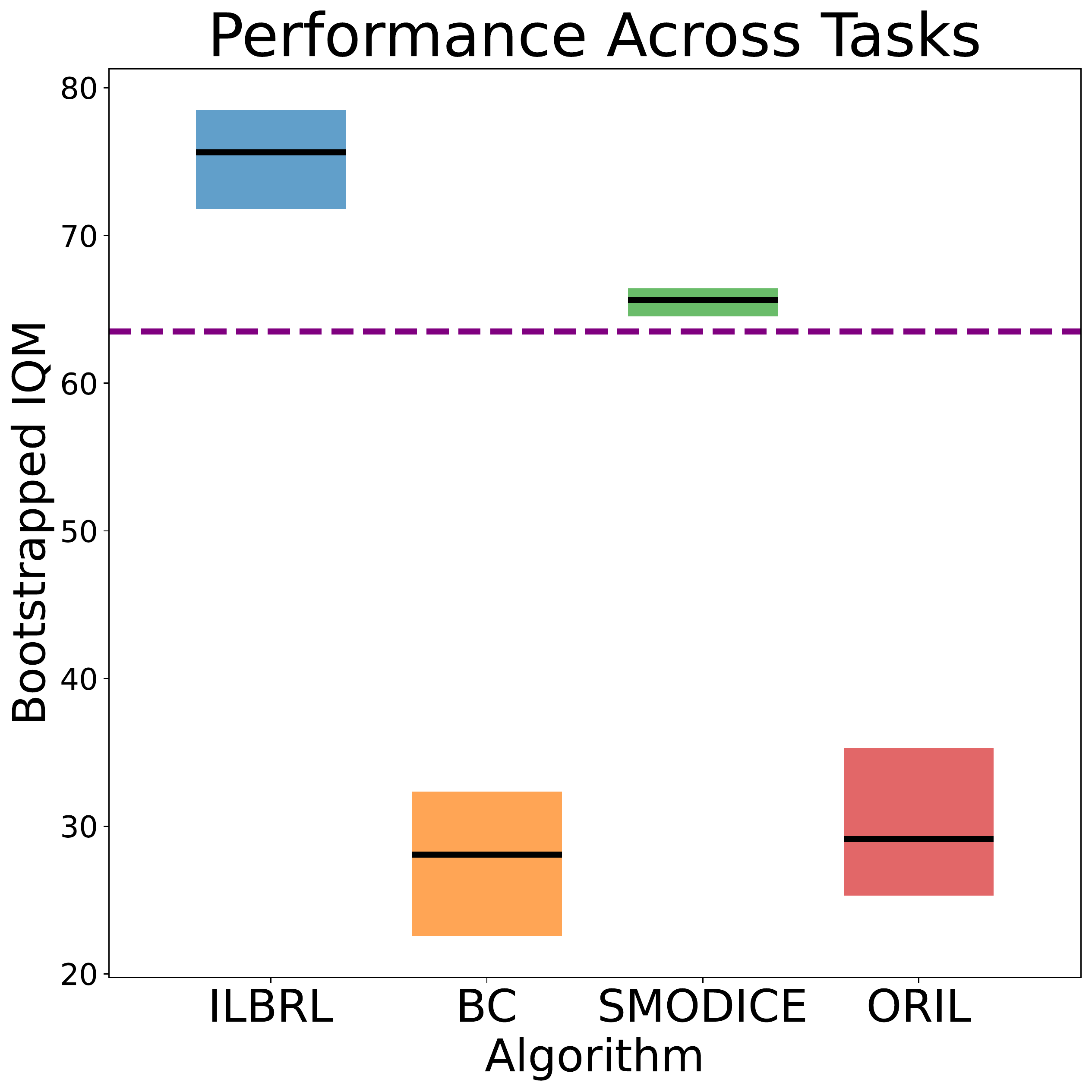}
    \caption{Aggregate Performance of ILBRL and baselines across D4RL tasks. Horizontal bars represent interquartile means. Shaded regions give 95\% confidence intervals based on bootstrapped samples \citep{agarwal2021deep}. The purple, dashed line gives the average performance of the exploratory policy. The expert policy achieves an expected score of 100.}
    \label{fig:bootstrap_CIs}
\end{figure}

\subsection{Continuous Control Tasks}

In many realistic RL settings, tabular function approximation is not possible, due to large or continuous state and action spaces. The Gym MuJoCo benchmark \citep{mujoco} is a standard set of challenging RL environments which cover such continuous control tasks. In order to extend our algorithm to these settings, we apply a modified version of the continuous support reward relaxation of \citet{ciosek2022imitation}. We employ the normalised reward:
\[
r(s,a) \coloneqq \max_{s', a' \in D_E}  1 - \frac{1}{d_{\textrm{max}}^{\frac{1}{2}}}\left\lVert \phi(s, a) - \phi(s', a') \right\rVert_2^{\frac{1}{2}},
\]
where
\[
d_{\textrm{max}} \coloneqq \max_{s, a \in D_U} \min_{s', a' \in D_E} \left\lVert \phi(s, a) - \phi(s', a')\right\rVert_2
\]
is the maximum distance observed in the dataset. The use of normalisation ensures that rewards are bounded in $[0,1]$ and helps control instability that arises in our context due to the need for offline deep RL methods, which generally lack convergence guarantees. We can think of this as a ``soft'' support estimator which assumes some smoothness over the environment and agent policy: state-actions that are in the expert dataset will have reward 1, just as in the tabular algorithm. However, in continuous environments it is unlikely that the same state will be visited more than once, so we also reward state-actions that are close to the expert data.

In our experiments, data comes from the D4RL MuJoCo benchmarks \citep{fu2020d4rl}, which are standard continuous control datasets used to benchmark state-of-the-art offline RL and IL methods. The D4RL data consists of several datasets, generated by policies of varying quality. The ``expert'' data represents rollouts from a stochastic agent trained on the true environment using SAC --- this policy is taken to be close to optimal in the environment. We use a fraction of the D4RL expert data as our $D_E$. In order to simulate the use of known safe exploratory policies for training and evaluation, our exploratory data, $D_X$, is derived from the D4RL ``medium'' dataset, which comes from a suboptimal, intermediate policy trained on the true environment. 

Our agent learns via a state-of-the-art offline RL method known as TD3-BC \citep{fujimoto2021minimalist}. The algorithm is a modification of the algorithm developed by \citet{fujimoto2018addressing}, with an additional BC regularisation term applied to the policy loss. This shifts the distribution induced by the learned policy to be closer to the data distribution and leads to more stable learning. We employ a small modification to the TD3-BC algorithm by multiplying the statewise BC loss by the intrinsic reward achieved in those states. This leads the agent to put more emphasis on imitating expert behavior, and less weight on the imitation of the exploratory policy. More details on the algorithm and architecture employed, as well as environments and data can be found in the Appendix.

\begin{figure}[ht]
    \centering
    \includegraphics[width=0.75\linewidth]{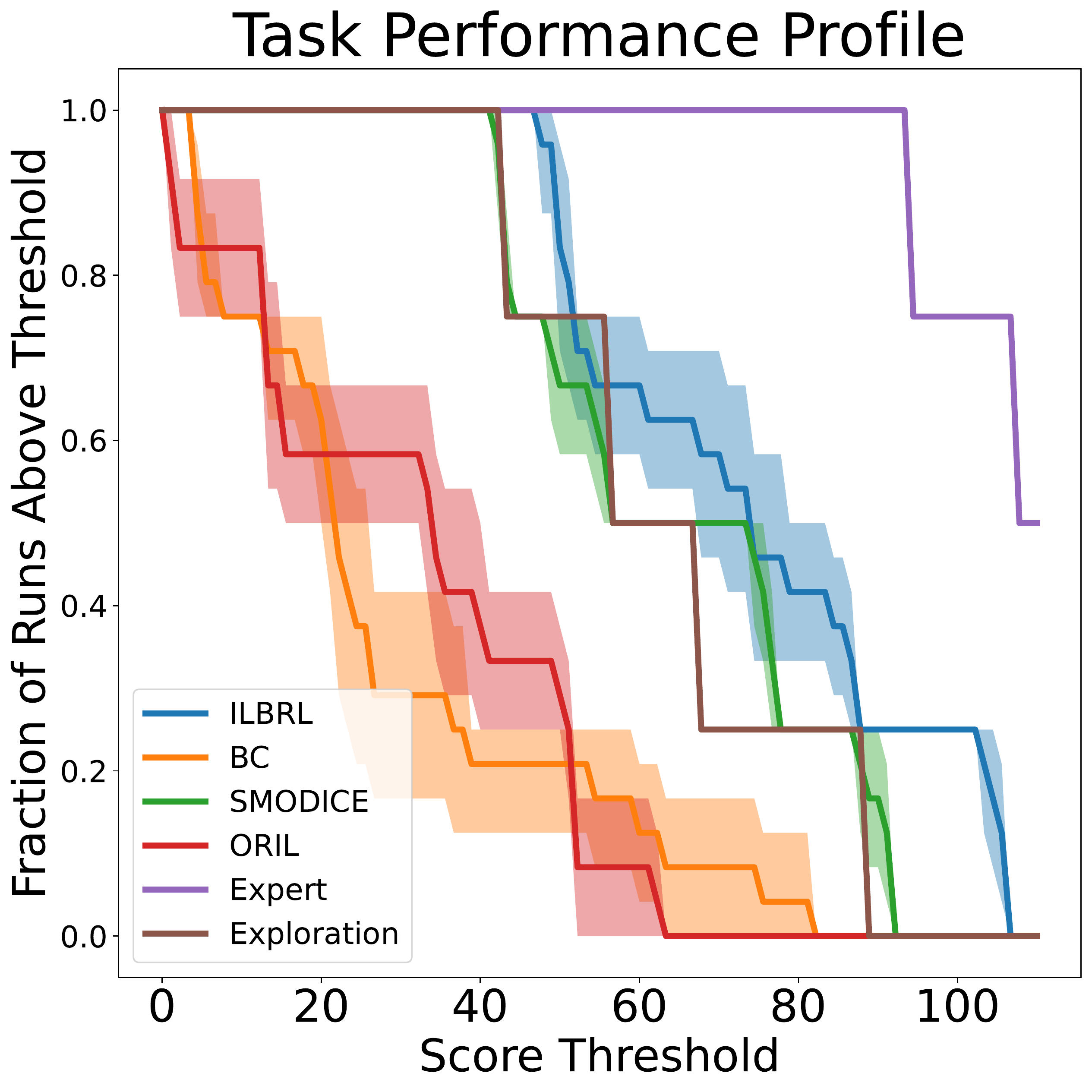}
    \caption{Performance profiles for ILBRL and baselines \citep{agarwal2021deep}. Curves that arc further from the origin are higher performing. ILBRL is competitive with state-of-the-art IL algorithms in difficult data regimes. It is able to achieve more tasks at a lower thresholds and reach higher performances when compared to the baselines.}
    \label{fig:performance_profiles}
\end{figure}

\paragraph{Baselines}
We compare to two state-of-the-art offline IL algorithms: SMODICE \citep{ma2022smodice} and ORIL \citep{zolna2020offline}. In addition, we compare to an expert-only behavioural cloning baseline, which makes use exclusively of $D_E$. For algorithms other than behavioural cloning, the training data mix was chosen such that, given equal quantities of expert data, the behavioural cloning baseline could not perform well. This stands in contrast to the dataset used by \citet{ma2022smodice}, which contains not just large amounts of expert data, but large amounts of random data as well, which confounds the BC learner. In some of our experiments, we found that giving the BC learner as much data as is reported by \citet{ma2022smodice} leads to near-optimal performance for some environments. We also provide a sample-based estimate of the D4RL policies, evaluated by rolling out the provided models on the environment several times. We use the argmax variant of the policies, which act according to the mean of their policy.

\paragraph{Results}

The final performance results of our algorithm, evaluated on the true environment after optimisation using our principled evaluation procedure can be found in \cref{fig:bootstrap_CIs} and \cref{fig:performance_profiles}. Our metrics were chosen based on the recommendations by \citet{agarwal2021deep}, for improved ease of comparison, and increased statistical robustness. Performance was measured according to the normalised D4RL score, which assigns random performance a score of 0, and expert performance a score of 100.  \cref{fig:bootstrap_CIs} represents 95\% confidence intervals (CI) for the interquartile mean (IQM) score of bootstrapped samples uniformly chosen across three rollouts per algorithm and all tasks. The IQM interpolates between the median and the mean statistic. It is estimated by first computing the mean normalised score of the middle 50\% of scores in a given bootstrap. Then the reported overall means and CIs are computed over the entire set of bootstrap means. \citet{agarwal2021deep} find that bootstrapped CIs are more statistically robust than the median, and less susceptible to outliers than the mean. Bootstrapping across tasks serves two purposes: it gives an aggregate notion of performance in order to efficiently compare algorithms, and allows for more statistically robust point estimates by making use of all available data at once.

We note that among the algorithms we tested, only ILBRL is able to outperform the exploratory policy by a statistically significant margin. This suggests that it is able to make good use of the expert data labels and is not simply cloning the exploratory policy. Because we are in a challenging, low expert data regime, performance of the baselines is somewhat lower than their reported values, which make use of large amounts of expert, or near-expert data. We hypothesise that our relative effectiveness is due to the simplicity of our approach, as well as the natural scaling of our method to imbalanced datasets, while the baselines rely on density estimation, which may be challenging under highly imbalanced data regimes.

In addition to our more robust approaches to evaluation, we report complete evaluation results with means and standard errors for all tests in \cref{tab:mean_table}. Our results carry over: our algorithm  outperforms or competes with the baselines on all of the environments.

For a more qualitative evaluation of our method, we make use of a performance profile, as shown in \cref{fig:performance_profiles}, also suggested by \citet{agarwal2021deep}. On the $x$ axis is a performance threshold, and on the $y$ axis is the fraction of runs across environments that achieved the level of performance specified on the $x$ axis. We see that our method is able to compete with expert performance on some tasks, as can be observed by the $x$-intercept reaching past a threshold of 100. Our algorithm strictly dominates the ORIL baseline, and reaches well beyond medium performance, showing that it is able to both perform very well on easier tasks, and maintain a reasonable performance for more challenging tasks.

\section{Conclusions}
In conclusion, we have developed and formally analysed a simple and effective algorithm for imitation learning from batch data. In addition, we have developed a novel approach to model selection and optimisation in the offline reinforcement learning setting. In the future, it would be meaningful to examine our evaluation protocol on other benchmarks, as well as extend our analysis beyond the tabular RL setting.

\bibliography{bibliography}
\bibliographystyle{icml2023}

\newpage
\appendix
\newtheorem{assumption}{Assumption}
\newtheorem{theoremalt}{Theorem}[theorem]
\newenvironment{theoremp}[1]{
	\renewcommand\thetheoremalt{#1}
	\theoremalt
}{\endtheoremalt}
\setcounter{theorem}{0}
\setcounter{corollary}{0}
\setcounter{lemma}{0}
\setcounter{claim}{0}
\newcommand{\algrule}[1][.2pt]{\par\vskip.3\baselineskip\hrule height #1\par\vskip.3\baselineskip}
\newcommand\Algphase[1]{%
\algrule
\hspace{-1.7\algorithmicindent}#1%
\algrule
}
\newcommand*{\skipnumber}[2][1]{%
  {\renewcommand*{\alglinenumber}[1]{}\State #2}%
  \addtocounter{ALG@line}{-#1}}

\onecolumn
\section{Proofs}
\label{appendix:theory}
In this section we provide detailed proof of the results that appear in the main paper. We also extend Theorem \ref{thm:regret} to include a bound on total variation distance between the expert and imitation policy, which follows directly from the regret bound with the same constant. We recall the setting of our work through the following key assumptions.

\begin{assumption}
    \label{assumption:deterministic}
    The expert and imitation policies are deterministic.
\end{assumption}
\begin{assumption}
    \label{assumption:ergodic}
   Expert, exploration, and imitation policies induce aperiodic and irreducible Markov chains.  
\end{assumption}
\begin{assumption}
    \label{assumption:pmin}
    The Markov chain induced by the exploratory policy has a uniform lower bound on the probability that the agent will visit any given state and action:
    \[\rho_{\pi_X}(s, a) > p_{\min}.\]
\end{assumption}
\begin{assumption}
    \label{assumption:kakade}
   The transition matrix obtained by following policy $\pi^*$ on the MDP, $P_{\pi^*}$ has $n$ distinct eigenvalues.
\end{assumption}
Assumption \ref{assumption:deterministic} cannot be relaxed as it is impossible to imitate a stochastic expert with a single call to the batch RL algorithm with a stationary reward function-optimisation may in general lead to a deterministic policy, as any stationary MDP has a deterministic optimal policy. Assumption \ref{assumption:ergodic} can, for the expert and imitation policies, possibly be relaxed to the periodic setting with appropriate attention to detail. However, in order to achieve appropriate sampling over the state space, this cannot be relaxed for the exploratory policy. Similar concerns suggest that Assumption \ref{assumption:pmin} cannot be relaxed without introducing other assumptions. Assumption \ref{assumption:kakade} is only needed to prove \cref{lemma:discounted_to_average_appendix}, and can, in principle, be relaxed.

\subsection{Theory of Phased Q-Learning}
We use the notation from \citet{kearns1998finite}:
\begin{itemize}
    \item $m$ is the number of samples from the parallel sampler,
    \item $\ell$ is the number of iterations of phased Q-learning, 
    \item  $s$ and $a$ are states and actions, and $P_{ss'}^a$ is the transition distribution,
    \item $r$ is a deterministic reward function,
    \item  $\hat{V}_i(\cdot), \hat{Q}_i(\cdot,\cdot)$ are our estimated value functions at iteration $i$, 
    \item  $\bar{V}_i(\cdot), \bar{Q}_i(\cdot,\cdot)$ are the true value functions of the policy greedy with respect to $\hat{Q}_i$,
    \item $V_i(\cdot), Q_i(\cdot,\cdot)$ is the output of $i$ rounds of policy iteration,
    \item  $V^*(\cdot), Q^*(\cdot,\cdot)$ are the optimal value functions.
\end{itemize}   

In addition, we use the following notation for different error variables, reproduced here for reference:
\begin{table}[H]
    \centering
    \begin{tabular}{|c|c|}
        \hline
        \textbf{Error Variable} & \textbf{Interpretation} 
        \\ \hline
        $\epsilon$ & Output error of \cref{algo:ilbrl} 
        \\ \hline
        $1 - \epsilon'$ &  Average intrinsic reward obtained by the learned policy
        \\ \hline
        $\epsilon''$ &  Reward construction error
        \\ \hline
        $1 - \nu$ &  Probability controllable intrinsic reward of expert policy.
        \\ \hline
        $\eta$ &  Discounted regret for phased Q-learning
        \\ \hline
        $\eta'$ &  Concentration error of sampled values
        \\ \hline
        
    \end{tabular}
    \caption{Error Variables}
    \label{tab:error_variables}
\end{table}
Following the intuition that phased Q-Learning is similar to value iteration, we make use of the following inequality--that our estimate of the bootstrapped value concentrates around the true bootstrap:
\begin{claim}
\label{claim:concentration}
For all $s, a$, and $i\leq\ell$, we have
\begin{align}
    \left\lvert 
        \frac{1}{m}\sum_{k=1}^m \hat{V}_i(s'_k) 
        - \sum_{s'}P_{ss'}^a \hat{V}_i(s')
    \right\rvert
    \leq \eta',
    \label{eqn:val_concentration_appendix}
\end{align}
where $s'_k$ is the $k$th sampled subsequent state from the parallel sampler at $s,a$.
\end{claim}

We will later take $m$ large enough to demonstrate that Claim \ref{claim:concentration} will hold with high probability. This is done in the proof of \cref{lemma:phased_q_appendix}.

\subsubsection{Proof Of Lemma 1}
\begin{lemma}
    \label{lemma:value_gap_appendix}
Given \cref{eqn:val_concentration_appendix}, the difference between the value output by phased Q-learning after $\ell$ steps and the optimal value is given by
\begin{align*}
\zeta(s,a)
    := \left\lvert \hat{Q}_{\ell}(s, a)  - Q^*(s,a) \right\rvert
    \leq \frac{\eta' \gamma + \gamma^{\ell}}{1 - \gamma}
\qquad \forall s, a.
\end{align*}
\end{lemma}

\begin{proof}
We begin by bounding the gap between value estimated by phased-Q iteration and the value output by policy iteration. Recall the phased Q-learning update, given by \[\hat{Q}_{i+1}(s, a) = r(s,a) + \gamma \frac{1}{m}\sum_{k=1}^m \hat{V}_i(s'_k).\]
For any $s,a$ pair, this means we have
\begin{align*}
&\left\lvert\hat{Q}_{i+1}(s, a)  - Q_{i+1}(s,a) \right\rvert \\
&\qquad = \Big\lvert r(s,a) + \gamma \frac{1}{m}\sum_{k=1}^m \hat{V}_i(s'_k)  - r(s,a) - \gamma \sum_{s'}P_{ss'}^a V_i(s') \Big\rvert \\
&\qquad \leq \gamma\Big\lvert
    \sum_{s'}P_{ss'}^a  \left(\hat{V}_i(s') - V_i(s')\right)
    \Big\rvert + \gamma \eta' \\
&\qquad \leq \gamma \max_{s'} \left\lvert 
        \hat{V}_i(s') - V_i(s')
    \right\rvert + \gamma \eta' \\
&\qquad \leq \gamma \max_{s'} \left\lvert 
         \max_{a'} \hat{Q}_{i}(s', a') - \max_{a '}Q_{i}(s', a')
    \right\rvert + \gamma \eta' \\
&\qquad \leq \gamma \max_{s', a'} \left\lvert 
         \hat{Q}_{i}(s', a') - Q_{i}(s', a')
    \right\rvert + \gamma \eta',
\end{align*}
where the first inequality makes use of our claim in (\ref{eqn:val_concentration_appendix}), the second uses the fact that we have a convex combination of the value gap over states, and the last from the max of a difference being greater than a difference of max.
Starting from $Q_0(s,a) = \hat{Q}_0(s,a)$ and recursively applying this equation leads to the bound
\begin{equation}
    \left\lvert\hat{Q}_{\ell}(s, a)  - Q_{\ell}(s,a) \right\rvert \leq \frac{\eta'\gamma}{1 - \gamma},
\end{equation}
by the geometric series and the fact that both algorithms are initialised identically. 

\citet[Thm 6.3.3]{puterman2005markov} shows that $\left\lvert Q_{\ell}(s, a)  - Q^*(s,a) \right\rvert \leq \gamma^\ell / (1 - \gamma)$.
Combining with the above yields
\begin{align*}
\left\lvert\hat{Q}_{\ell}(s, a)  - Q^*(s,a) \right\rvert
    &\leq \frac{\eta'\gamma  + \gamma^\ell}{1 - \gamma}.
\qedhere
\end{align*}
\end{proof}

\subsubsection{Proof of Lemma 2}
Recall that the discounted regret is defined as
\[R_\gamma(i) := \mathbb{E}_{s_0}\left[\bar{V}^*(s_0) - \bar{V}_i(s_0) \right].\]

\begin{lemma}
    \label{lemma:regret_appendix}
    Given \cref{eqn:val_concentration_appendix}, the policy output by phased Q-learning after $l$ steps achieves discounted regret of at most:
    \[ R_\gamma(\ell) \leq \frac{2}{(1-\gamma)}
        \left(\frac{\eta' \gamma  + \gamma^l}{1 - \gamma}\right).\]
\end{lemma}

\begin{proof}
We begin by bounding the statewise suboptimality of the learned policy using the suboptimality of a single action. Let $\pi_\ell(s) := \textrm{argmax}_a \hat{Q}_\ell(s,a)$ give the policy output after $\ell$ steps of phased Q-learning.\
We abuse notation slightly here by writing $a := \pi_\ell(s)$, since context is clear. Then we have
\begin{align*}
&\left \lvert \bar{V}_\ell(s) - V^{*}(s) \right \rvert = \Big \lvert
    r(s, a) + \gamma \sum_{s'}P_{ss'}^a \bar{V}_\ell(s') - V^{*}(s)
    \Big \rvert, \\
&\qquad = \Big \lvert
        \gamma \sum_{s'}P_{ss'}^a \big[ \bar{V}_\ell(s') - V^{*}(s') \big]
        + Q^*(s, a) - V^{*}(s)
    \Big \rvert, \\
&\qquad \leq
    \gamma \Big \lvert \max_{s'} \left[ \bar{V}_\ell(s') - V^{*}(s') \right] \Big \rvert
    + \Big\lvert Q^*(s, a) - V^{*}(s) \Big \rvert.
\end{align*}
Unrolling this gives rise to a geometric series in $\gamma$:
\begin{align}
\label{eqn:single_action_gap}
\begin{split}
&\lvert \bar{V}_\ell(s) - V^{*}(s) \rvert  \leq \frac{1}{1-\gamma} \max_s \left\lvert Q^*(s, \pi_{\ell}(s)) - V^{*}(s) \right\rvert.
\end{split}
\end{align}
Let $\pi^*(s) := \argmax_a Q^*(s,a)$ be the optimal policy. We bound this gap for $a = \pi_\ell(s) \neq \pi^*(s) = a^*$, since otherwise it is zero.
Since, by construction, $\hat{Q}_{\ell}(s, a^*) \leq  \hat{Q}_{\ell}(s, a)$, the gap is maximised when we are both underestimating the optimal value of the optimal action and overestimating the optimal value of a suboptimal action:
\begin{align*}
\big\lvert Q^*(s, a) - V^{*}(s) \big\rvert
&=
     Q^{*}(s, a^*) - Q^*(s, a),
\\
&= 
        Q^{*}(s, a^*)
        -\hat{Q}_{\ell}(s, a^*)
        +\hat{Q}_{\ell}(s, a^*)
        -\hat{Q}_{\ell}(s, a)
        +\hat{Q}_{\ell}(s, a)
        - Q^*(s, a),
\\
&\leq
        \lvert Q^{*}(s, a^*)
        -\hat{Q}_{\ell}(s, a^*) \rvert
        +\hat{Q}_{\ell}(s, a^*)
        -\hat{Q}_{\ell}(s, a)
        + \lvert \hat{Q}_{\ell}(s, a)
        - Q^*(s, a) \rvert,
\\
&\leq 2 \left(\frac{\eta'\gamma  + \gamma^\ell}{1 - \gamma}\right) + \underbrace{\hat{Q}_{\ell}(s,a^*) - \hat{Q}_{\ell}(s, a)}_{\le 0},
\end{align*}
where the final line is an application of \cref{lemma:value_gap_appendix}.
Since the both this and \cref{eqn:single_action_gap} apply across all states, combining this with \cref{eqn:single_action_gap} gives the bound on our regret.
\end{proof}

\subsubsection{Proof of Lemma 3}
We recall Hoeffding's Inequality:
\begin{exlemma}[Hoeffding's Inequality]
    \label{exlemma:hoeffding}
    Let $X_j$ be a sequence of i.i.d. random variables uniformly bounded by $0 \leq X_j \leq a$. Define the sample mean of the sequence of length $m$ as $\bar{X}_m := \frac{1}{m}\sum_{j=1}^m X_j$. Then we have
    \[
    P\left(
        \left \lvert \bar{X}_m 
        - \mathbb{E}[X_1]\right \rvert \geq t
    \right)
    \leq 2\exp\left(
        -2\frac{t^2 m}{a^2}
    \right).\]
\end{exlemma}

Using Hoeffding's Inequality, once we have solved for the minimum acceptable concentration error $\epsilon'$, we can solve for $m$, the number of samples we need for \cref{eqn:val_concentration_appendix} to hold with high probability. This is accomplished in the following lemma.

\begin{lemma}
\label{lemma:phased_q_appendix}
Taking $m \ell$ samples from the parallel model for phased Q-learning, such that
\begin{align*}
m \ell
    \geq \log\left(\frac{2 \ell SA}{\delta'}\right)\frac{2 \gamma^2 \ell}{(1-\gamma)^2(\eta(1 - \gamma)^2 - 2\gamma^l)^2},
\end{align*}
we obtain discounted regret less than $\eta$ with probability at least $1 - \delta'$.
\end{lemma}

\begin{proof}
In order to bound the overall error by $\eta$, we choose the phased Q-learning concentration error as

\begin{align*}
\eta' = \left[
    \eta(1 - \gamma)^2 / 2 - \gamma^\ell
\right] / \gamma, 
\end{align*}
with the additional condition that:
\begin{align}
\label{eqn:ell}
\ell \geq \frac{\log \eta + 2 \log (1 - \gamma) - \log 2}{\log \gamma},
\end{align}
to ensure that $\eta'$ is not negative.

Substituting our choice of $\eta'$ into Lemma~\ref{lemma:regret_appendix}, we see that regret obeys
 \[ R_\gamma(\ell) \leq \frac{2}{(1-\gamma)}
        \left(\frac{\eta' \gamma  + \gamma^l}{1 - \gamma}\right) = \frac{2}{(1-\gamma)}
        \left(\frac{
            \frac{
                \eta(1 - \gamma)^2 / 2 - \gamma^\ell
            }{\gamma}
            \gamma  + \gamma^l}{1 - \gamma}\right)
        = \eta.
\]

With our error bounded, we proceed to bound the probability with which we fail to achieve this error, which will lead to our overall bound on $m$. We define the concentration error as
\begin{align*}
    \Delta_i(s, a) := \left\lvert 
        \frac{1}{m}\sum_{k=1}^m \hat{V}_i(s'_k) 
        - \sum_{s'}P_{ss'}^a \hat{V}_i(s')
    \right\rvert.
\end{align*}
Next, we apply Hoeffding's Inequality for each $s, a, i$, substituting $X_j = \hat{V}_i(s'_k)$ and $t=\eta'$. At step $i$, conditioned on $s$ and $a$, $\hat{V}_i$ is a deterministic function with range $[0, (1-\gamma)^{-1}]$ and the $s'_k$ are independent draws from the MDP transition function at $s, a$, so we obtain
\begin{align*}
&P\left(\Delta_i(s,a) \geq \eta'\right) \leq 2 \exp \left \{
    -\frac{
        2m\left([\eta(1 - \gamma)^2 / 2 - \gamma^\ell] / \gamma \right)^2
    }{
        (1-\gamma)^{-2}
    }
\right \}.
\end{align*}

We take a union bound over state, action, and algorithm step, since the overall bound needs to hold for all of the above. Then, bounding the resulting failure probability by $\delta'$ gives
\begin{align*}
2 S A \ell \exp \left \{
    -\frac{
        2m\left([\eta(1 - \gamma)^2 / 2 - \gamma^\ell] / \gamma \right)^2
    }{
        (1-\gamma)^{-2}
    }
\right \}
\leq
\delta'.
\end{align*}
Solving for $m$ leads to
\begin{align*}
m \geq \log\left(\frac{2 \ell S A}{\delta'}\right)\frac{2\gamma^2}{(1-\gamma)^2(\eta(1 - \gamma)^2 - 2\gamma^l)^2}
\end{align*}
and the claim follows by multiplying both sides by $\ell$.
\end{proof}

\begin{remark}
Observing that, for fixed $\gamma$, $\ell = O( \log \eta)$, we have $m\ell = O(\eta^{-2} [\log(SA / \delta') +\log (1 / \eta) + \log \log(1 / \eta)])$, matching the bound of \citet[Theorem. 1]{kearns1998finite} up to an additive term $\eta^{-2} \log A$.
\end{remark}

\subsubsection{Proof of Lemma 4}
\cref{lemma:lucas_lemma_appendix} combined with \cref{lemma:phased_q_appendix} give the overall sample complexity of the phased-Q learning stage of our algorithm.
\begin{lemma}
\label{lemma:lucas_lemma_appendix}
The number of sample transitions from the exploration policy needed to perfectly simulate the parallel sampling model with failure probability less than $\delta''$ is upper bounded by
\begin{align*}
\frac{2 t^{\pi_X}}{\log(2) p_{\rm min}}  \log \left( \frac{2}{p_{\rm min}} \right)  \log \left( \frac{SA}{\delta''} \right).
\end{align*}
\end{lemma}
\begin{proof}
Let $\rho^T$ be the distribution after rolling out the policy for $T$ steps from any starting distribution.
By \citet[Sec. 4.5]{levin2017markov}, we know that
$T \ge  t^{\pi_X} \log(1 / \tau) / \log(2)$ implies
$\lVert \rho^T - \rho^{\pi_X}\rVert_{TV} \leq \tau$.
We choose
\begin{align*}
T \ge t^{\pi_X} \log(2 / p_{\rm min}) / \log(2),
\end{align*}
such that $\tau \le p_{\rm min} / 2$.
Assume that we collect a sample every $T$ steps.
After collecting $N$ samples, the probability that a given state-action pair is never observed is bounded from above by $(1 - p_{\rm min} + \tau)^N \le \exp(-N p_{\rm min} / 2)$.
Bounding this probability by $\delta''$, taking a union bound over all state-action pairs and solving for $N$ yields
\begin{align*}
N \ge 2 / p_{\rm min} \log(SA / \delta'').
\end{align*}
The claim is obtained by multiplying $N$ and $T$.
\end{proof}

\subsection{Bounding Total Variation Distance Between Expert and Imitator}
This section contains the proof of  Theorem 1, as well as the needed auxiliary lemmas.
\subsubsection{Proof of Lemma 5}

\begin{lemma}[Relaxation of Theorem 1 from \citet{kakade2001optimizing}]
    \label{lemma:discounted_to_average_appendix}
    Suppose policy $\pi$ is $\epsilon$-suboptimal or better at each state with respect to the discounted return. Let $\pi^*$ be the optimal policy under the average reward criterion.  Taking $\Sigma$ to be the matrix of right eigenvectors of $P_{\pi^*}$ with corresponding eigenvalues: $\lambda_1  = 1 > ... \geq \vert \lambda_n \vert$ Then, under Assumption \ref{assumption:kakade}, we have
    \[
        \mu^{\pi} \geq \mu^{\pi^*} - \kappa(\Sigma)\lVert \mathbf{r} \rVert \frac{1 - \gamma}{1 - \gamma \lvert \lambda_2 \rvert} - (1 - \gamma) \epsilon,
    \]
    where $\kappa(\Sigma) := \lVert \Sigma \rVert_2 \lVert \Sigma^{-1} \rVert_2$ is the condition number of a matrix, and $\mathbf{r}$ is the vector with entries $r(s, \pi^{*}(s))$ for each state.
\end{lemma}
\begin{proof}
Our proof closely follows that of \citet{kakade2001optimizing}. Let $\pi^{\gamma^*}$ be the optimal policy under the discounted return criterion. Then we have, for all $s$
\[V^{\pi^{\gamma^*}}(s) \geq V^{\pi^*}(s).\]
Subtracting $\epsilon$ from both sides gives
\begin{align*}
    V^{\pi}(s) \geq V^{\pi^{\gamma^*}}(s) - \epsilon \geq V^{\pi^*}(s) - \epsilon,
\end{align*}
for all states, where the first inequality follows from the assumption that $\pi$ is $\epsilon$-suboptimal or better. From here we make use of the relationship between average reward and discounted reward:
\begin{align*}
    \mu^{\pi} & = (1 - \gamma)\sum_s \rho_{\pi}V^{\pi}(s),
    \\
    & \geq (1 - \gamma)\sum_s \rho_{\pi}(s)( V^{\pi^*}(s) - \epsilon),
    \\
    & \geq (1 - \gamma)\sum_s \rho_{\pi}(s)( V^{\pi^*}(s)) -  (1 - \gamma)\epsilon,
    \\
    &\geq \mu^{\pi^*} - \kappa(\Sigma)\lVert \mathbf{r} \rVert \frac{1 - \gamma}{1 - \gamma \lvert \lambda_2 \rvert} - (1 - \gamma) \epsilon.
\end{align*}
The second inequality comes from the fact that $\rho_{\pi}$ sums to one across states. The final inequality follows exactly from the proof in \citet{kakade2001optimizing}.
\end{proof}

\subsubsection{Proof of Theorem 1}

\hfill

First we reproduce two key lemmas from \citet{ciosek2022imitation} for convenience.
\begin{lemma}[Lemma 5, \citet{ciosek2022imitation}]
\label{lemma:ciosek_5_appendix}
 Given an expert dataset consisting of $\lvert D_E \rvert$ points, 
 a policy, $\pi$ which maximises the average intrinsic reward achieves an expected average intrinsic reward of at least
 \[ \mu_{\textrm{int}}^\pi := 1 - \nu - \sqrt{\frac{8St^{\pi_E}}{\lvert D_E \rvert}},\]
 for all error terms $\nu > 0$, with probability at least
 \[1 - \delta''' := 1 - 2 \exp\left(-\frac{\nu^2 \lvert D_E \rvert}{4.5 t^{\pi_E}}\right).\]
 
\end{lemma}
 
\begin{lemma}[Lemma 7, \citet{ciosek2022imitation}]
\label{lemma:ciosek_7_appendix}
An agent which achieves an average reward of $1 - \epsilon'$ on the intrinsic reward MDP also achieves average reward of
\[(1 - \epsilon')\mu^{\pi_E} - 4 t^{\pi_E} \epsilon,\]
on the true MDP.
\end{lemma}

Recall that $\mu^{\pi}$ is the average per-step reward achieved by a policy $\pi$ on the true MDP. Per-step imitation regret with respect to an expert policy $\pi_E$ is then defined as
\[R_{\pi_E}(\pi) = \mu^{\pi_E} - \mu^{\pi},\]
and the steady state distribution of a policy, $\pi$, over state, action pairs is given by $\rho_\pi$. $\mu_{\textrm{int}}^{\pi}$ gives the average reward obtained on the intrinsic reward MDP by policy $\pi$.
\begin{theorem}
\label{thm:regret}
Consider an instance of algorithm 1. Under assumptions \ref{assumption:deterministic}-\ref{assumption:kakade},
in order to achieve regret on the average reward problem of less than $\epsilon$, or total variation distance between the expert and the imitation policy less than $\epsilon$, with probability $1 - \delta$, we need to sample 
\[\max\left\{\frac{1}{\epsilon^2}32St^{\pi_E}(1+4t^{\pi_E})^2, \frac{4.5t^{\pi_E}16(1 +  4 t^{\pi_E})^2 \log(4 / \delta)}{\epsilon^2}\right\}\]
transitions from the expert policy, and 
\[\mathcal{O}\left(
\frac{t^{\pi_X}}{p_{\rm min}}  \log \left( \frac{1}{p_{\rm min}} \right)  \log^2 \left( \frac{SA}{\delta} \right)
    \log\left(
        \frac{t^{\pi_E}\beta}{\epsilon(1 - \lambda)}
    \right)
    \frac{(t^{\pi_E})^8\beta^6}{\epsilon^8 (1 - \lambda)^6}
\right)\]
from the exploratory policy, where $\beta =  \kappa(\Sigma)\lVert \mathbf{r} \rVert$, as described in \cref{lemma:discounted_to_average_appendix}.
\end{theorem}
\begin{proof}
Let
\begin{equation}
    1 - \epsilon' := \mathbb{E}_{s,a \sim \rho_{\pi}}[\hat{r}(s, a)] = \mu_{\textrm{int}}^\pi
    \label{eqn:mu_internal}
\end{equation}represent the average reward achieved by the policy learned on the intrinsic reward problem. We begin by directly evoking \cref{lemma:ciosek_7_appendix}. This leads to our overall regret on the true problem, given by
\begin{align*}
    \mu^{\pi_E} - \mu^{\pi} & \leq \epsilon'\mu^{\pi_E} + 4 t^{\pi_E} \epsilon',
    \\
    & \leq \epsilon' \left(1 + 4 t^{\pi_E}\right).
\end{align*}
We then wish to bound the number of samples needed to achieve $\epsilon'$ small enough such that the overall error obeys
\begin{align}
    \epsilon' \left(1 + 4 t^{\pi_E}\right) 
    \leq \epsilon,
    \label{eqn:desired_inequality}
\end{align}
with probability $\delta$.

Lemma \ref{lemma:discounted_to_average_appendix} gives us the means to decompose $\epsilon'$ in terms of the error incurred by batch learning and construction of the intrinsic reward problem in our algorithm. Let the policy learned by phased Q-learning be $\eta$-suboptimal with respect to the discounted intrinsic return across all states. Lemma \ref{lemma:discounted_to_average_appendix} allows us to move to the intrinsic reward setting, by bounding the distance between the average intrinsic reward obtained by our learned policy and the optimal intrinsic average reward using this quantity. However the optimal intrinsic average reward itself is a function of the amount of data we have. Lemma \ref{lemma:ciosek_5_appendix} bounds this additional error incurred from a potential lack of expert data. Specifically, we define the reward construction error as
\begin{align*}
\epsilon'' := \nu + \sqrt{\frac{8St^{\pi_E}}{\lvert D_E \rvert}}.
\end{align*}

From \cref{lemma:ciosek_5_appendix}, we have $\mu_{\textrm{int}}^{\pi^*} \geq 1 - \epsilon''$, with some probability $1 - \delta'''$, which we will return to bound, once $\epsilon''$ has been bounded. Substituting $\mu_{\textrm{int}}^{\pi^*} = 1 - \epsilon''$ and $\epsilon = \eta$ in \cref{lemma:discounted_to_average_appendix} we have:
\begin{align*}
\mu_{\textrm{int}}^\pi \geq 1 - \nu - \sqrt{\frac{8St^{\pi_E}}{\lvert D_E \rvert}} - \kappa(\Sigma)\lVert \mathbf{r} \rVert \frac{1 - \gamma}{1 - \gamma \lvert \lambda_2 \rvert} - (1 - \gamma) \eta.
\end{align*}
From the definition of $\epsilon'$ in \cref{eqn:mu_internal},  this leads to the following bound
\begin{align*}
\epsilon' \leq  
    (1-\gamma)
    \frac{\kappa(\Sigma) \lVert \mathbf{r} \rVert}
         {1 - \gamma \lvert \lambda_2 \rvert}
    + (1-\gamma)\eta
 + \nu + \sqrt{\frac{8St^{\pi_E}}{\lvert D_E \rvert}},
\end{align*}
which we set to
\begin{align*}
    \underbrace{
        (1-\gamma)
        \frac{\kappa(\Sigma) \lVert \mathbf{r} \rVert}
             {1 - \gamma \lvert \lambda_2 \rvert}
    }_{\textrm{T1}}
    + \underbrace{
        (1-\gamma)\eta
    }_{\textrm{T2}}
    + \underbrace{
        \nu
    }_{\textrm{T3}} 
    +  \underbrace{
        \sqrt{\frac{8St^{\pi_E}}{\lvert D_E \rvert}}
    }_{\textrm{T4}} 
 \leq \frac{\epsilon}{(1 + 4 t^{\pi_E})},
\end{align*}
in order to satisfy \cref{eqn:desired_inequality}.

Since there are four terms here with largely independent algorithmic parameters, which we refer to as $\textrm{T1} - \textrm{T4}$, we allow error to be distributed evenly across all the terms, which allows us to bound the above expression termwise.

In order to control the first term, we need to set $\gamma$ large enough to satisfy
\[\textrm{T1} \leq \frac{1}{4} \frac{\epsilon}{(1 + 4 t^{\pi_E})}. \]
For convenience we define $\alpha := 4(1 + 4 t^{\pi_E}) / \epsilon$ and $\beta := \kappa(\Sigma) \lVert \mathbf{r} \rVert$. This leads us to
\begin{align}
\frac{(1-\gamma)\beta}{1 - \gamma \lvert \lambda_2 \rvert} \leq \frac{1}{\alpha}.
\label{eqn:gamma_boundary}
\end{align}
We set
\begin{align}
\gamma = \frac{\alpha \beta - 1}{\alpha \beta - \lvert \lambda_2 \rvert}, 
\label{eqn:gamma_constraint}
\end{align}
or, equivalently
\begin{align*}
(1 - \gamma) = \frac{1 - \lvert \lambda_2 \rvert}{\alpha \beta - \lvert \lambda_2 \rvert}.
\end{align*}
in order to satisfy \cref{eqn:gamma_boundary}.
Setting $\gamma$ in this way thus allows us to minimise the error incurred by transferring from the discounted to the average reward setting, by increasing the effective horizon of the discounted problem.

Moving on to the second term, directly choose
\begin{align*}
(1 - \gamma)\eta = \frac{1}{\alpha},
\end{align*}
to satisfy $\textrm{T2} \leq 1/\alpha$.
We leave the error in this form, since $\eta$ is multiplied by a factor of $1 - \gamma$ in \cref{lemma:phased_q_appendix}. Intuitively, this corresponds with the scaling of average per step rewards in the discounted problem by a factor of $1 / (1 - \gamma)$. 

Continuing to the terms corresponding to the definition of our intrinsic reward problem, we choose
\begin{align*}
\nu = \frac{1}{\alpha},
\end{align*}
to satisfy $\textrm{T3} \leq 1/\alpha$.
Finally, we choose
\begin{align}
\lvert D_E \rvert = \left\lceil \frac{1}{\epsilon^2}128St^{\pi_E}(1+4t^{\pi_E})^2 \right \rceil,
\label{eqn:expert_1}
\end{align}
where $\lceil{\cdot}\rceil$ is the ceiling function which chooses the smallest integer greater than the argument. This allows us to satisfy
\begin{align*}
\textrm{T4} = \sqrt{\frac{8St^{\pi_E}}{\lvert D_E \rvert}} \leq \frac{1}{\alpha}.
\end{align*}
This provides us with our first bound on the number of expert samples, and completes the bounding of overall error by the relevant factors. To complete our proof, we must ensure that these errors hold with overall probability $\delta$.

With the error of our algorithm controlled, we move on to bounding the failure probability of our algorithm. Let $\delta'''$  be the maximum probability of the expert policy failing to achieve an intrinsic average reward of $1 - \epsilon''$ as in \cref{lemma:ciosek_5_appendix}. From \cref{lemma:ciosek_5_appendix}, substituting $1/\alpha$ for $\nu$, we have
\[\delta''' \leq 2 \exp\left(
    -\frac{
        \left(\frac{\epsilon}{4(1 +  4 t^{\pi_E})}\right)^2\lvert D_E \rvert
    }{4.5t^{\pi_E}}
\right).\]
Setting this to be less than $\delta / 2$ and rearranging leads to the bound
\begin{align}
\lvert D_E \rvert \geq \frac{72t^{\pi_E}(1 +  4 t^{\pi_E})^2 \log(4 / \delta)}{\epsilon^2}.
\label{eqn:expert_2}
\end{align}
This is our second bound on $\lvert D_E \rvert$. Taking the maximum over \cref{eqn:expert_1} and \cref{eqn:expert_2} gives our overall bound on the number of expert samples needed.

To bound the failure probability of phased Q-learning, we first bound the probability that our simulation of the parallel sampler fails. Invoking \cref{lemma:lucas_lemma_appendix} with the substitution $\delta'' = \delta / 4$, we get
\begin{align}
N \geq \frac{2 t^{\pi_X}}{\log(2) p_{\rm min}}  \log \left( \frac{2}{p_{\rm min}} \right)  \log \left( \frac{SA}{\delta / 4} \right).
\label{eqn:explore_1}
\end{align}

With our remaining error budget, we can look to bound the probability that phased Q learning fails. Substituting $1/\alpha$ for $(1- \gamma)\eta$ and $\delta'=\delta/4$ in \cref{lemma:phased_q_appendix} gives
\begin{align}
    m\ell 
    \geq \log\left(\frac{2 \ell SA}{\delta / 4}\right)\frac{2 \gamma^2 \ell}{(1-\gamma)^2(\frac{1}{\alpha}(1 - \gamma) - 2\gamma^{\ell})^2}.
    \label{eqn:explore_2}
\end{align}
To complete the proof, we recall from \cref{eqn:ell} that
\begin{align}
    \ell 
    \geq \log_{\gamma}\left(\frac{\frac{1}{\alpha}(1 - \gamma)}{2}\right).
    \label{eqn:explore_3}
\end{align}
Together, \cref{eqn:expert_2,eqn:explore_1,eqn:explore_2,eqn:explore_3,eqn:gamma_constraint} give the sample complexity needed to bound the failure probability of our procedure, given by $\delta' + \delta'' + \delta'''$, below the global error $\delta$, with the overall number of expert samples given by the maximum over \cref{eqn:expert_1} and \cref{eqn:expert_2}, and the number of exploratory samples given by $\lvert D_X \rvert \geq N\ell m$.

While the above gives the most precise version of our bound, it is difficult to interpret the effect of changing $\epsilon, \delta$, and the parameters of the environment, largely due to the obscuring influence of $\gamma$ on $m$ and $\ell$.

To fully simplify, we will choose sensible values for $\ell$ and $\gamma$, which will lead to a bound with a simpler form. Since $\gamma < 1$, $\log_\gamma$ is a decreasing function, we choose
\begin{align*}
    \ell = \log_{\gamma}\left(\frac{\frac{1}{\alpha}(1-\gamma)}{4}\right).
\end{align*}
Substituting this for $\ell$ in \cref{eqn:explore_2} leads to the following bound on $m$:
\begin{align}
    m
    &\geq \log\left(
        \frac{2SA\log_{\gamma}\left(
            \frac{\frac{1}{\alpha}(1-\gamma)}{4}\
        \right)}{\delta / 4}
    \right)
    \frac{2 \gamma^2 }{(1-\gamma)^2\left(\frac{(1 - \gamma)}{\alpha} - \frac{(1-\gamma)}{2\alpha}\right)^2} \nonumber \\
    & \geq \log\left(
        \frac{2SA\log_{\gamma}\left(
            \frac{\frac{1}{\alpha}(1-\gamma)}{4}\
        \right)}{\delta / 4}
    \right)
    \left(\frac{8 \gamma^2 \alpha^2}{(1 - \gamma)^4}\right).
    \label{eqn:m_l_substitute}
\end{align}
To further simplify, we choose a sensible value for $\gamma$. Noticing that our bound on $\gamma$ based on \cref{eqn:gamma_boundary} is increasing in $\alpha\beta$, instead of \cref{eqn:gamma_constraint}, we choose
\begin{align*}
    \gamma = \frac{2\alpha\beta - 1}{2\alpha\beta - \lvert \lambda_2 \rvert}.
\end{align*}
Which, since $2\alpha\beta \geq 1$ and $0 \leq \lvert \lambda_2 \rvert \leq 1$, remains in $[0, 1)$, and satisfies \cref{eqn:gamma_boundary}.
Substituting this into our choice of $\ell$ and simplifying leads to
\begin{align*}
    \ell = \frac{
        \log\left(
            \frac{\epsilon^2 (1 - \lvert \lambda_2 \rvert)}
            {4(1 + 4 t^{\pi_E})(8(1 + 4 t^{\pi_E})\beta - \lvert \lambda_2 \rvert \epsilon)}
        \right)
    }{
         \log\left(
            \frac{8(1 + 4 t^{\pi_E})\beta - \epsilon}
            {8(1 + 4 t^{\pi_E})\beta - \lvert \lambda_2 \rvert \epsilon}
        \right)
    }.
\end{align*}
If we divide $\ell$ by
\[l = \left(\frac{4(1 + 4 t^{\pi_E})\beta}{\epsilon (1 - \lvert \lambda_2 \rvert)}\right)^2,\]
and take limits, through use of a limit solver, we find that multiple applications of l'Hopital's Rule give \[
\lim_{\alpha\to\infty, \beta\to\infty, \lambda\to 1} \frac{\ell}{l} = 0, 
\]
which implies that:
\[
\ell = \mathcal{O}\left(\left(\frac{4(1 + 4 t^{\pi_E})\beta}{\epsilon (1 - \lvert \lambda_2 \rvert)}\right)^2\right).
\]

All that is left then is to bound the second factor in \cref{eqn:m_l_substitute}. Substituting our values of $\gamma$, we have
\begin{align*}
    \left(\frac{8 \gamma^2 \alpha^2}{(1 - \gamma)^4}\right) &= 
    8\alpha^2\frac{(2\alpha\beta - 1)^2(2\alpha\beta - \lvert \lambda_2 \rvert)^2}{(1 - \lvert \lambda_2 \rvert)^4} \\
    & \leq 8\alpha^2\frac{(2\alpha\beta)^4}{(1 - \lvert \lambda_2 \rvert)^4},
\end{align*}
where the second inequality follows from the fact that $2 \alpha \beta \geq 1$.
Bringing this together with our first term, this gives an overall bound for $m$ as
\[m = \mathcal{O}\left(
    \log\left(
         \frac{SA\left(\frac{t^{\pi_E}\beta}{\epsilon(1 - \lvert \lambda \rvert)}\right)}{\delta}
    \right)
    \frac{(t^{\pi_E})^6\beta^4}{\epsilon^6 (1 - \lambda)^4}
\right)\]
Multiplying with our bounds for $\ell$ and $N$, we get the overall sample complexity of the exploratory dataset
\[\lvert D_X \rvert = Nm\ell = \mathcal{O}\left(
\frac{t^{\pi_X}}{p_{\rm min}}  \log \left( \frac{1}{p_{\rm min}} \right)  \log^2 \left( \frac{SA}{\delta} \right)
    \log\left(
        \frac{t^{\pi_E}\beta}{\epsilon(1 - \lambda)}
    \right)
    \frac{(t^{\pi_E})^8\beta^6}{\epsilon^8 (1 - \lambda)^6}
\right).\]

The bound on total variation  follows from the well-known identity
\[\mathbb{E}_{\rho_1}\left[f\right] - \mathbb{E}_{\rho_2}\left[f\right] \leq \epsilon \implies \left\lVert \rho_1 - \rho_2 \right\rVert_{TV} \leq \epsilon, \]
for two discrete distributions, $\rho_1, \rho_2$ and a function $f$ with range $[0,1]$. See e.g. \citet{ciosek2022imitation} for a proof. Substituting $\rho_{\pi_E}$ for $\rho_{1}$, $\rho_{\pi}$ for $\rho_{2}$, with $f = r$, gives
\begin{align}\mu_{\pi_E} - \mu_{\pi} \leq \epsilon \implies \lVert \rho_{\pi_E} - \rho_{\pi}\rVert_{TV} \leq \epsilon.\end{align}
Choosing $\lvert D_E \rvert$ and $\lvert D_X \rvert$ sufficiently high as described above leads to the left side holding true as discussed. This implies that with the same number of samples, the total variation distance between expert and imitator is less than $\epsilon$, with probability at least $1 - \delta$.
\end{proof}

\section{Details of Evaluation Protocol}

\subsection{Pseudocode}
Algorithm \ref{algo:evaluation_protocol} represents an abstract procedure by which offline RL practitioners can tune and evaluate agents that are trained on offline data, without access to the true environment. The protocol integrates an all-important phase in which the offline policy evaluation (OPE) algorithm itself is tuned, a practice which we found to be highly necessary, and absent from existing literature altogether.

We view an offline RL algorithm, $\mathcal{A}$ as a (stochastic) function which maps from a dataset, $D \in \mathcal{D}$, and hyperparameter configuration, $\phi \in \Phi$, to a policy $\pi \in \Pi$, where $\mathcal{D}$ is the space of all sets of $(s, a, r, s')$-tuples. Our protocol is then used to optimise over a discrete set of candidate hyperparameter configurations, returning those that lead to the highest expected performance on held-out data. Here we assume that datasets are indexed by trajectory in the first dimension, then state, though a shuffled dataset can be used, so long as initial and terminal states are labeled appropriately.

An OPE algorithm then maps from datasets and policies, as well as OPE hyperparameters $\varphi \in \varPhi$ to value functions $V \in \mathcal{V}$. Our criterion for selection of OPE hyperparameters is referred to as $\textrm{RankError}_D$, which takes a set of value functions learned by OPE, $V_k$, each of which evaluates a different policy, $\pi_k$, and compares how the ranking implied by $V_k$ compares to sample estimates of $V^{\pi_k}$
\[\textrm{RankError}_D(\{V_k\}, \{V^{\pi_k}\}) 
    = \sum_{k}\left\lvert 
        \textrm{Rank}\left[V_{k}(D)\right] 
        - \textrm{Rank}\left[V^{\pi_k}(D)\right]
    \right \rvert.
\]
In our case, rankings are performed over value functions that are averaged over several random seeds, since the OPE algorithm uses randomly shuffled data. In the case where OPE diverged, the corresponding term was set to the maximum $\lvert\{\pi_k\}\rvert$, with other rankings adjusted to minimise $\textrm{RankError}$. Ties were broken based on the total distance to the true values $\sum_k \lvert V_k(D) - V^{\pi_k}(D) \rvert$. $D_{0:i}$ refers to subset of $D$ with indices between 0 and $i$, ${D}_{l:\lvert D \rvert, 0}$ is the subset of $D$ with indices between $l$ and $\lvert D \rvert$ intersected with the subset of $D$ corresponding to start states.

\subsection{Tuning Data Composition for Policy Evaluation}

In our experiments, we found that tuning the policy evaluation hyperparameters played a vital role in ensuring good policy selection. Indeed, for many hyperparameter settings, policy evaluation diverged altogether. In our setting, since the data was derived from multiple policies, one key parameter was the relative mix of different sources on which to evaluate the policy. This is illustrated in Figure \ref{fig:p_eval}.

\subsection{Algorithmic Benefits of Tuning Protocol}
We note that an additional algorithmic benefit of this protocol is that, due to the instability of policy evaluation under distribution shift, our policy evaluation step represents a form of implicit regularisation on the final policy selected for deployment. Policies that lead to instability in evaluation are likely to diverge heavily from the distribution induced by the dataset, while policies that stay close to the data will lead to more stable evaluation. Since the limitation of distribution shift seems to be fundamental to current practice in batch RL, we propose that this is beneficial, especially in the context of imitation learning, in which we are looking to achieve a policy similar to the one used to generate the data.

\setcounter{algorithm}{1}
\begin{center}
\begin{minipage}{.8\linewidth}
\centering
\begin{algorithm}[H]
\caption{Offline Hyperparameter Tuning for RL \label{algo:evaluation_protocol}}
\begin{algorithmic}[1]
    \renewcommand{\algorithmicrequire}{\textbf{Input:}}
    \renewcommand{\algorithmicensure}{\textbf{Output:}}
    \REQUIRE \mbox{}\newline
        \vspace{-1.2em}
        \begin{description}
            \item[Data:] Dataset $D$, Training set fraction $d$, 
            \item[Offline RL:] Algorithm $\mathcal{A}:\mathcal{D} \times \Phi \to \Pi$,  Hyperparameters $\{\phi_n\}$, Random seeds $\{x_s\}$
            \item[OPE:] Algorithm $\mathcal{A}': \mathcal{D} \times \Pi \times\varPhi \to \mathcal{V}$, Training fraction $d'$, Known policies $\{\pi_k\}$, Hyperparameters $\{\varphi_j\}$
        \end{description} 
    \ENSURE Optimised hyperparameters, $\phi^*$, Performance of optimised algorithm $R_{\mathcal{A}(D, \phi^*)}$
    \Algphase{Phase 1: Data Splitting}
        \STATE $i \gets \textrm{int}(d\lvert D \rvert )$
        \STATE $D_T \gets D_{0:i}$ \hfill  $\triangleright$ Assign training set
        \STATE $D_V \gets D_{i:\lvert D \rvert}$ \hfill  $\triangleright$ Assign validation set
        \STATE $l \gets \textrm{int}( d'\lvert D_V \rvert )$
        \STATE $D_{V_{\textrm{PE}}} \gets {D_V}_{0:l}$ \hfill  $\triangleright$ Assign policy evaluation training set
        \STATE $D_{V_{F}} \gets {D_V}_{l:\lvert D_V \rvert, 0}$ \hfill $\triangleright$ Assign final performance validation set (initial states only)
    \Algphase{Phase 2: Offline Policy Evaluation Tuning}
        \FOR{$\varphi_j \in \{\varphi_j\}$}
            \FOR{$\pi_k \in \{\pi_k\}$}
                \STATE $V_{jk} \gets \mathcal{A}'(D_{V_{\textrm{PE}}}, \pi_k, \varphi_j)$ \hfill  $\triangleright$ Evaluate known policy on PE training set
            \ENDFOR
        \ENDFOR
        \STATE  $\varphi^* \gets \argmin_{j} \textrm{RankError}_{ D_{V_{F}}}(\{V_{jk}\}, \{V^{\pi_k}\})$  \hfill  $\triangleright$ Choose evaluation hparams from initial states
    \Algphase{Phase 3: Offline RL Training}
        \FOR{$\phi_n \in \{\phi_n\}$}
            \FOR{$x_s \in \{x_s\}$}
                \STATE $\pi_{ns} \gets \mathcal{A}(D_{T}, \phi_n)$ \hfill  $\triangleright$ Train policy via offline RL for several seeds
            \ENDFOR
        \ENDFOR
    \Algphase{Phase 4: Hyperparameter Selection}
        \FOR{$\pi_{ns} \in \{\pi_{ns}\}$}
            \STATE $V_{ns} \gets \mathcal{A}'(D_{V_{\textrm{PE}}}, \pi_{ns}, \varphi^*)$ \hfill  $\triangleright$ Evaluate policy with optimal PE hyperparameters
        \ENDFOR
        \STATE $\phi^* \gets \argmax_{n} \sum_{s} V_{ns}(D_{V_F})$ 
        \hspace{0.01cm} \hfill  $\triangleright$ Optimal hyperparameters maximise value of initial states
        \STATE $R_{\mathcal{A}(D, \phi^*)} \gets \textrm{EvalEnv}({\pi_{\phi^*}})$ \hfill  $\triangleright$ Evaluate optimal hyperparameters averaged across seeds
        \RETURN $\left(\phi^*, R_{\mathcal{A}(D, \phi^*)}\right)$
\end{algorithmic}
\end{algorithm}
\end{minipage}
\end{center}

\begin{figure}
    \centering
    \includegraphics[width=0.4\textwidth]{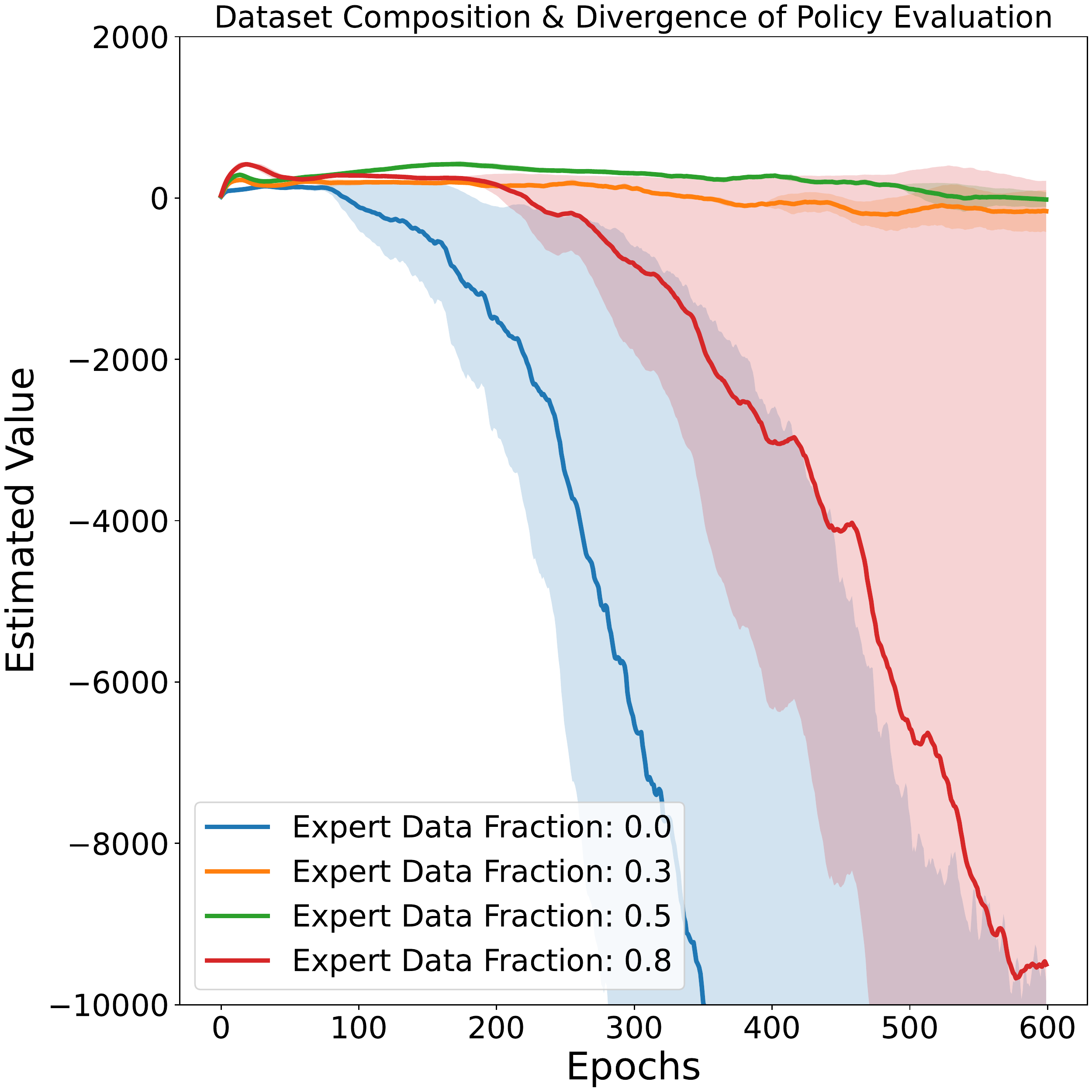}
    \caption{Policy Evaluation of a single policy on varying dataset compositions. We see that for more evenly mixed data, evaluation is more stable, while for less balanced data, divergence occurrs. This highlights the importance of tuning the policy evaluation protocol itself.}
    \label{fig:p_eval}
\end{figure}

\section{Details of Empirical Analysis}
Here we provide a complete overview of the empirical evaluation of our algoritm and the baselines. 
\subsection{Environments}
We make use of the OpenAI Gym MuJoCo \citep{mujoco} continuous control benchmarks. These locomotion tasks involve learning control policies for several robots with varying mythologies in a simulated physical environment. The agents, which are restricted to movement in a two dimensional plane, are optimised for forward velocity and maintaining their centre of mass in a prespecified vertical range (representing the agent maintaining an upright pose). Specifically, we make use of \texttt{Hopper-v2}, \texttt{Halfcheetah-v2}, \texttt{Walker2d-v2}, and \texttt{Ant-v2} as D4RL datasets \citep{fu2020d4rl} (described below) are available for these environments. We use \texttt{gym} version \texttt{0.24.1}, and MuJoCo version \texttt{2.1}. Additional details regarding reward function specifics, as well as state and action spaces can be found in the documentation at \href{https://github.com/openai/gym}{https://github.com/openai/gym}. A discount factor of $\gamma = 0.99$ was used across all experiments.
\subsection{Datasets}
We make use of the D4RL dataset suite \citep{fu2020d4rl} for training and offline policy evaluation. Seveal datasets are available for each of the four environments listed above, each generated by a policy of different quality: random, medium, or expert. Policies were trained via Soft Actor Critic (SAC) \citep{haarnoja2018soft} online on the environment. The ``medium'' datasets are taken from an intermediate, suboptimal policy which has not yet converged. SAC learns a stochastic policy from which the datasets are sampled. Each dataset represents $10^6$ sample transitions from an environment of the form $\langle s, a, r, s', t, d\rangle$, which is a normal state transition, plus binary indicators for reaching a terminal state as well as reaching $1000$ time steps, after which the episode is cut short artificially. In addition, each dataset provides the parameters of the policy used to generate the data. For imitation learning algorithms, we mask the true reward, save for the evaluation protocol described above. Specifically, we use the version of D4RL available at the commit: \href{https://github.com/rail-berkeley/d4rl/commit/d842aa194b416e564e54b0730d9f934e3e32f854}{https://github.com/rail-berkeley/d4rl/commit/d842aa194b416e564e54b0730d9f934e3e32f854}.
\subsubsection{Dataset Composition}
In order to simulate use of known, safe policies for realistic evaluation of offline IRL methods, we make use of the ``medium'' D4RL data for our exploratory dataset, and the ``expert'' D4RL data for our expert. Noting that a naive behavioral cloning (BC) baseline using only the expert data was able to achieve an expert-level policy for several environments when $100$ expert trajectories were used, we instead opted to use only $18$ expert trajectories, which we found to be the maximum number such that the BC baseline began to perform suboptimally. Our total training dataset size was then taken to be $6\textrm{e}5$ transitions, in order to leave enough data to be used for policy evaluation. Since each expert trajectory was approximately $1000$ timesteps long, this means that the training datasets were composed of only $3\%$ expert data. For policy evaluation, as described above, the dataset composition was tuned as a hyperparameter. We found that, from values we tried, a $50/50$ split was optimal. The overall size of the policy evaluation set was $7\textrm{e}5$ transitions, with $5\textrm{e}5$ used for training, and $2\textrm{e}5$ reserved for validation, of which only $208$ were the initial states ultimately used for validation in our experiments.
\subsection{Algorithms}
In this section we comprehensively cover the details of the algorithms used in our experiments. For baselines, we made use of lightly modified versions of implementations made available in association with \citet{ma2022smodice}, which can be found at: \href{https://github.com/JasonMa2016/SMODICE}{https://github.com/JasonMa2016/SMODICE}
\subsubsection{Architectures}
All policies and value functions across our method and the baselines shared a common architecture, taken from \citet{ma2022smodice}, but established as standard in several works prior. Policies and values were parameterised as multilayer perceptrons \citep{Goodfellow-et-al-2016}. Two hidden layers were used, each of size \texttt{256}. Rectified Linear Units (ReLU) \citep{Goodfellow-et-al-2016} were used as nonlinearities between layers. For policies, since MuJoCo environments have symmetrically bounded action spaces, the outputs of these networks were then passed through an elementwise \texttt{tanh} function scaled to match the environment bounds. The discriminator network used for SMODICE \citep{ma2022smodice} and ORIL \citep{zolna2020offline} baselines was an MLP of the same size as the policy and value networks, though here \texttt{tanh} nonlinearities were used between layers.

\subsubsection{Algorithm Implementation Details}
Both baselines and ILBRL were implemented using the PyTorch scientific computing library (\href{https://pytorch.org/}{https://pytorch.org/}), version \texttt{1.12.0}. Experiments were run on Google Cloud Compute \texttt{n1-standard-64} VM instances with \texttt{Intel Haswell} CPUs and the \texttt{Debian 10} operating system. GPU computing was not used for our experiments. Open source code will be made available upon publication.

As was suggested in \citet{fujimoto2021minimalist}, and used in \citet{ma2022smodice}, we employed dataset standardisation. Sample means, $\mu_{D_T}$ and standard deviations $\sigma_{D_T}$ were estimated for observations $o$ across the training sets, and all data passed through the networks, $s$, was transformed according to $s = (o - \mu_{D_T}) / \sigma_{D_T}$.

The Adam \citep{kingma2014adam} optimiser was used for optimisation across all learned functions. Parameters were kept at the defaults used by PyTorch version \texttt{1.12.0}, except for the learning rate (gain) which was tuned as a hyperparameter. The baselines made use of a reward scaling parameter, which was also tuned as a hyperparameter.

For ILBRL, we made use of TD3-BC \citep{fujimoto2021minimalist} as our offline RL algorithm, though in principle, other algorithms designed for a similar setting could work as well. TD3-BC is currently considered a state-of-the-art algorithm, and was chosen for its simplicity, as it only introduces one additional hyperparameter as compared to an online RL algorithm, the fact that it learns deterministic policies, and the demonstrated effectiveness of the method. We found that this combination made TD3-BC ideal for our setting, which demanded simplicity especially due to our multi-stage evaluation protocol described above. TD3-BC can be summarised as the policy update scheme:
\[
    \theta' = \theta + \alpha \nabla_{\theta} \left(
        \lambda Q(s,\pi_\theta{s}) 
        + \textrm{BC}(\pi_{\theta}(s), a)
    \right),
\]
where $s,a$ are the state and action in the dataset, $\alpha$ is a learning rate in $[0,1]$, and the $Q$ function is learned separately through Double-Q learning on the dataset, with clipped zero-mean noise added to the policy when performing the bootstrapped update. The BC term is simply the mean-squared-error between the action taken by the policy and the action in the dataset. We introduced one change to the TD3-BC algorithm, which was a re-weighting of the statewise behavioural cloning loss:
\[
    \textrm{BC}_{\textrm{\algo}}(s, a'):= \textrm{BC}(\pi(s), a)r_I(s,a),
\]
where $r_I$ is our intrinsic reward function. This re-weighting encourages imitation of the expert policy, and reduces the degree to which the exploration policy is copied. We found this led to higher performance based on our offline evaluation protocol. An ablation study can be seen below in Figure \ref{fig:ablation_plots}.

\subsubsection{Off Policy Evaluation Implementation Details}
All methods were evaluated using the Expected SARSA \citep{van2009theoretical} off-policy evaluation algorithm with target networks \citep{mnih2013playing}. Stochastic policies were evaluated according to the policy $\pi'(s,a) = \mathbb{E}\left[\pi(s,a)\right]$, which is common practice for Gaussian-parameterised policies such as the ones employed here \citep{haarnoja2018soft}. The expected SARSA bootstrapped update for deterministic policies with target networks is then given by:
\[
\theta' = \theta + \nabla_\theta \left[ 
    Q_{\theta}(s,a) - r(s,a) - \gamma Q_{t}\big(s', \pi(s')\big)
\right]^2,
\]
with target networks updated using the exponential moving average:
\[
    t' = (1 - \tau)t + \tau\theta
\]
In order to make policy evaluation more robust, updates were performed according to the mean bootstrap values of two separately parameterised value and target networks: in the equations above this corresponds to:
\[
\theta_i' = \theta_i + \nabla_{\theta_i} \left(
    Q_{\theta_i}(s,a) - r(s,a) - \frac{1}{2}\gamma \left[
        Q_{t_1}\big(s', \pi(s')\big) + Q_{t_2}\big(s', \pi(s')\big)
    \right]
\right)^2,
\]
for $i\in\{1,2\}$, with target networks updated as before in accordance with their corresponding value network.

\subsubsection{Evaluation Protocols}
In order to tune hyperparameters, we made use exclusively of the protocol described above, using only single evaluations on the environment after hyperparameter ranges had been decided in advance. Each offline RL algorithm was evaluated across 10 hyperparameter configurations, with 6 random \textit{policy seeds} per hyperparameter. Policy evaluation was then performed for each of the random seed-hyperparameter combinations. Since policy evaluation is itself a stochastic algorithm, evaluation of each configuration was tested across 3 \textit{evaluation seeds}. Amongst evaluations that converged, the hyperparameter configuration that was then chosen was the one that had the highest learned value on a held-out dataset consisting only of environment start states.

\paragraph{Reported Results}
The performance of the optimal configuration was evaluated for each of the \textit{policy seeds} as the mean return of of 30 rollouts of 1000 steps on the true environment, averaged across 10 different random \textit{environment seeds}. These means were taken to represent the true performance of the learned policy.

This gives $6$ \textit{policy performance} data points per algorithm, per environment (one for each policy seed). To aggregate this data, we follow the recommendations of \citet{agarwal2021deep}. We perform a stratified bootstrap, sampling $6$ points with replacement from the policy performance pool for each environment. To compute the IQM of a bootstrap, we discard samples in the bottom or top 25th percentile of the $6$ samples in each environment. The mean is then taken over the remaining data: across environments for our confidence interval plots, and per environment for the performance profile.

In order to compute confidence intervals, we repeat this process $10\ 000$ times, taking the lower bound of the interval to be the data point at the $2.5$th percentile of IQM samples, and the upper bound to be the data point at the $97.5$th percentile. This corresponds to the \textit{percentile bootstrap assumption}

The environmentwise results in \cref{tab:mean_table} were computed as the mean and 95\% CI scores of the three policy seeds per environment, using the more conventional \textit{normal assumption}.

\subsection{Additional Results}

In order to examine in more details the effects of our design decisions, we have conducted additional experiments. To demonstrate the importance of tuning the policy evaluation protocol, we performed policy evaluation on a randomly selected policy trained on our dataset using ILBRL in the \texttt{Walker2d} environment. Varying only the data composition of the policy evaluation process, we observed that when the data was less balanced between exploratory and expert data, policy evaluation was more likely to diverge altogether, while more evenly mixed data led to stable evaluation. These results can be found in Figure \ref{fig:p_eval}.

In addition, we performed an ablation study on our modified reward function. We found it to improve performance, though not across all tasks. A performance profile and IQM bootstrap CI can be found in figure \ref{fig:ablation_plots}. In addition, despite the fact that gains are not extreme, our policy evaluation procedure is able to select for the reward-weighted ILBRL over the ablation using only offline data, further evidencing the robustness and effectiveness of our evaluation procedure.

\begin{figure}[H]
    \centering
    \begin{subfigure}{0.3\textwidth}
        \centering
         \includegraphics[width=\textwidth]{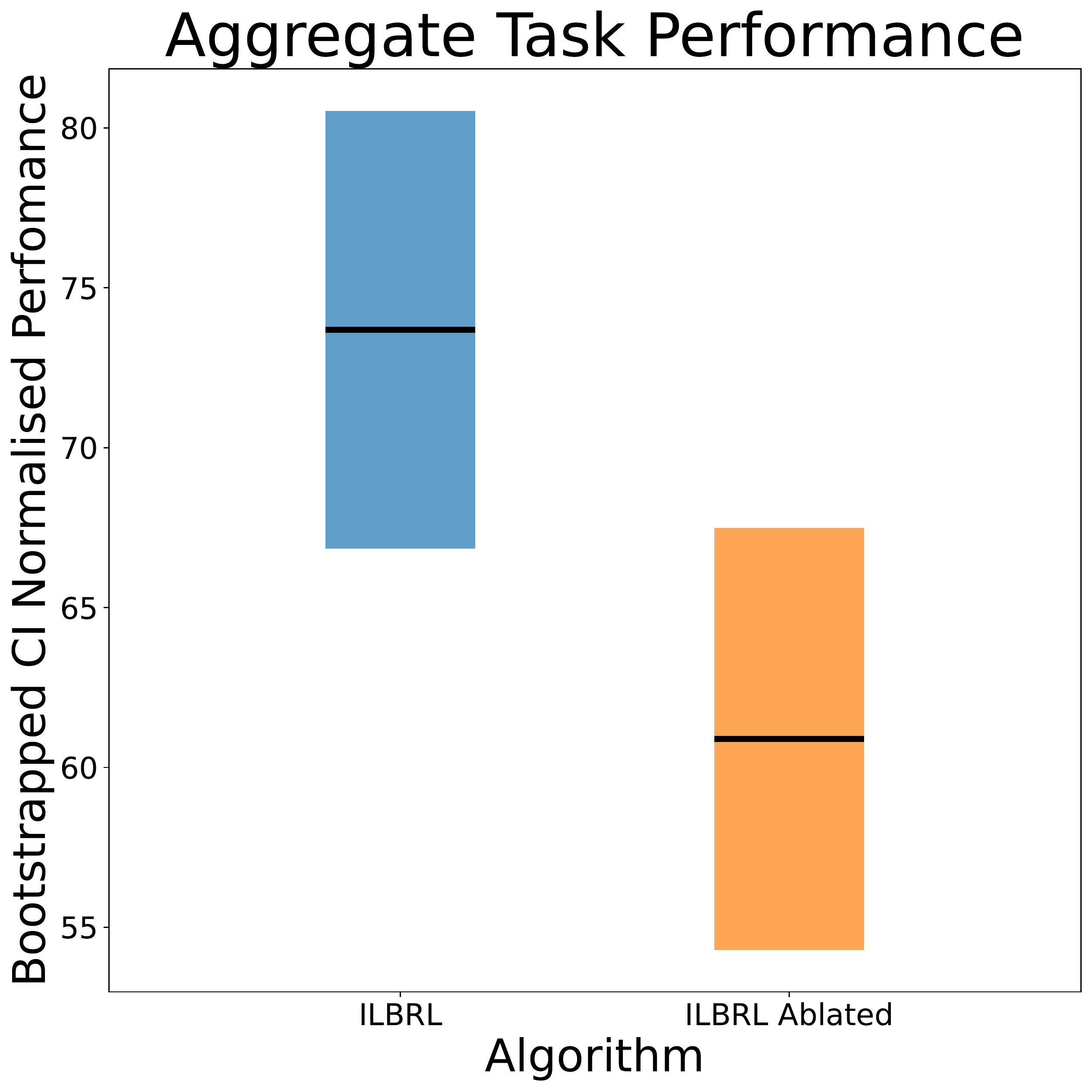}
    \end{subfigure}
    \hfill
    \begin{subfigure}{0.3\textwidth}
        \centering
         \includegraphics[width=\textwidth]{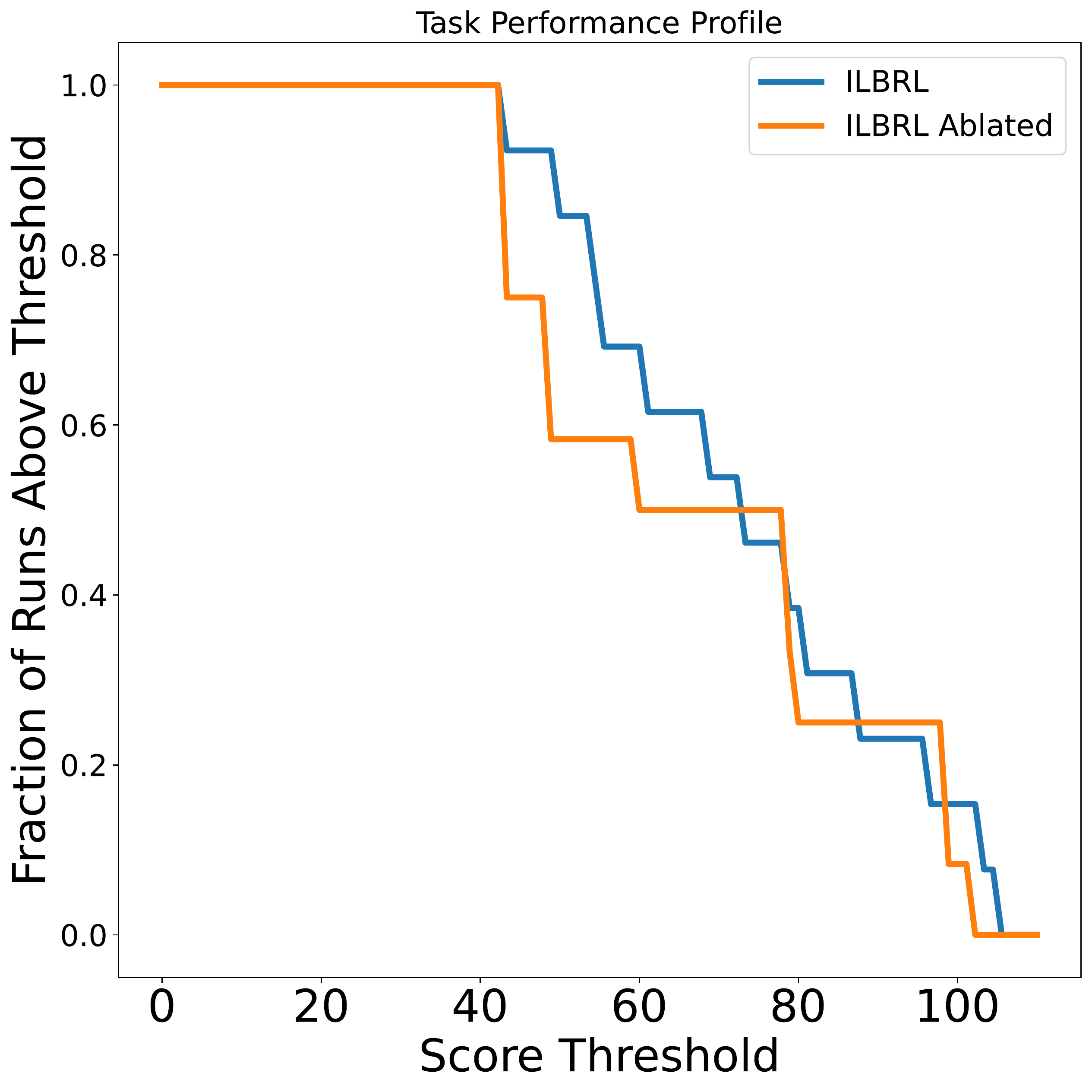}
    \end{subfigure}
    \hfill
    \begin{subfigure}{0.3\textwidth}
        \centering
         \includegraphics[width=\textwidth]{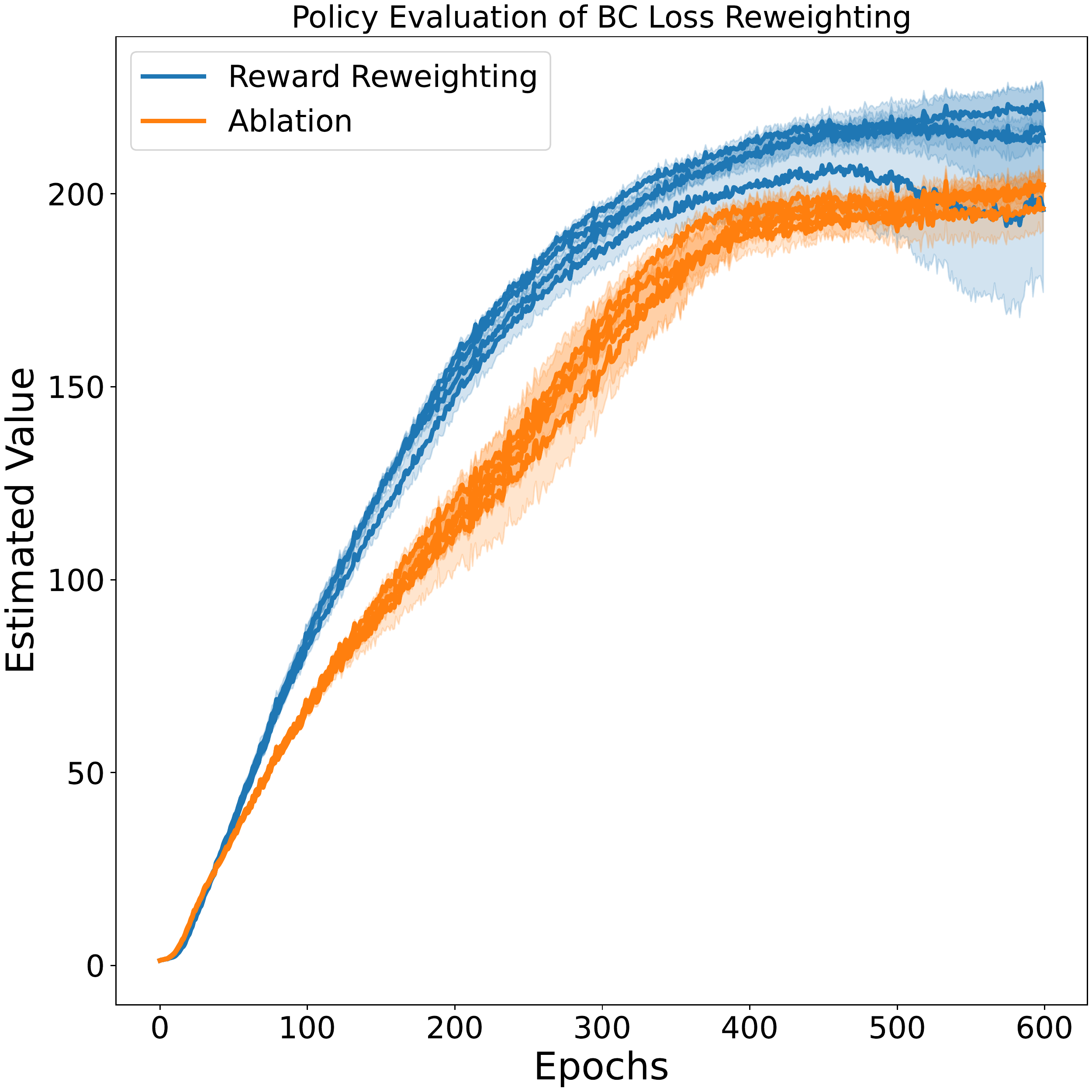}
    \end{subfigure}
  \caption{Performance of ILBRL with reward-weighted BC loss and without reward weighting (ablated). We see that reward weighting does provide some performance gains on the true environment, (left and centre figures), and that these gains are reflected in our offline policy evaluation procedure (right). \label{fig:ablation_plots}}
\end{figure}

\subsection{Hyperparameters}
\subsubsection{Tuning Ranges}
All hyperparameters were tuned using random search. In order to evaluate fairly against the baselines, given the novel evaluation protocol, we retuned baseline hyperparameters, with ranges near the defaults from \href{https://github.com/JasonMa2016/SMODICE}{https://github.com/JasonMa2016/SMODICE}. For ILBRL, hyperparameters that were not tuned were taken to be defaults from the TD3-BC implementation at \href{https://github.com/sfujim/TD3_BC}{https://github.com/sfujim/TD3\_BC}. We emphasise that policy evaluation was tuned using policies from the dataset, not from our method or the baselines.

\begin{table}[H]
    \centering
    \begin{tabular}{|l|c|}
        \hline
        \textbf{Hyperparameter} & \textbf{Range or Setting} \\ \hline
        $\lambda$ & uniform(1,4)\\
        Critic Learning Rate & uniform($10^{-5}, 10^{-4}$)\\
        Actor Learning Rate & uniform($10^{-5}, 10^{-4}$)\\
        Training Epochs & 400\\
        Minibatches Per Epoch & 2500\\
        Minibatch Size &  256\\
        \hline
    \end{tabular}
    \caption{ILBRL Hyperparameter Ranges}
    \label{tab:ILBRL_Ranges}
\end{table}
\begin{table}[H]
    \centering
    \begin{tabular}{|l|c|}
        \hline
        \textbf{Hyperparameter} & \textbf{Range or Setting} \\ \hline
        Reward Scale & uniform(0.05, 3)\\
        Discriminator Learning Rate & uniform($10^{-4}, 10^{-3}$)\\
        Actor Learning Rate & uniform($10^{-5}, 10^{-4}$)\\
        \hline
    \end{tabular}
    \caption{SMODICE \& ORIL Hyperparameter Ranges}
    \label{tab:SMODICE_Ranges}
\end{table}
\begin{table}[H]
    \centering
    \begin{tabular}{|l|c|}
        \hline
        \textbf{Hyperparameter} & \textbf{Range  or Setting} \\ \hline
        Learning Rate & loguniform($3 \times 10^{-4}, 3 \times 10^{-2}$)\\
        Training Epochs & 400\\
        Minibatches Per Epoch & 2500\\
        Minibatch Size &  256\\
        \hline
    \end{tabular}
    \caption{Behavioural Cloning Hyperparameter Ranges}
    \label{tab:BC_Ranges}
\end{table}
\begin{table}[H]
    \centering
    \begin{tabular}{|l|c|}
        \hline
        \textbf{Hyperparameter} & \textbf{Range} \\ \hline
        Expert Dataset Fraction & choice($[0, 0.3, 0.5, 0.8]$)\\
        Target Network Update Frequency & choice($[2, 10, 16]$)\\
        Value Learning Rate & choice($[3\times 10^{-6}, 3\times 10^{-5}, 3\times 10^{-4}, 3\times 10^{-3}]$)\\
        \hline
    \end{tabular}
    \caption{Off Policy Evaluation Hyperparameter Ranges}
    \label{tab:OPE_Ranges}
\end{table}
\subsubsection{Final Hyperparameters}
In this section we detail the hyperparameters selected for use by the policy evaluation process. We note that for reproducibility reasons, the random samples used during training were seeded identically, so the configurations for \texttt{Halfcheetah}, \texttt{Walker2d}, and \texttt{Ant} are identical for ILBRL, all environments ended up selecting the same hyperparameters for ORIL, and hyperparameters are identical for \texttt{Hopper} and \texttt{Walker2d} with BC.
\begin{table}[H]
    \centering
    \begin{tabular}{|l|c|}
        \hline
        \textbf{Hyperparameter} & \textbf{Value} \\ \hline
        Expert Dataset Fraction & $0.5$\\
        Target Network Update Frequency & $16$\\
        Value Learning Rate & $3\times 10^{-4}$\\
        \hline
    \end{tabular}
    \caption{Policy Evaluation Hyperparameters}
    \label{tab:OPE_hyperparameters}
\end{table}
\begin{table}[H]
    \centering
    \begin{tabular}{|l|c|}
        \hline
        \textbf{Hyperparameter} & \textbf{Value} \\ \hline
        $\lambda$ & $1.9898232924151003$\\
        Critic Learning Rate & $6.308344747206642\times 10^{-5}$\\
        Actor Learning Rate & $4.279938571320915\times 10^{-5}$\\
        \hline
    \end{tabular}
    \caption{ILBRL Hyperparameters: Hopper}
    \label{tab:ILBRL_Hopper}
\end{table}
\begin{table}[H]
    \centering
    \begin{tabular}{|l|c|}
        \hline
        \textbf{Hyperparameter} & \textbf{Value} \\ \hline
        $\lambda$ & $3.780707811143581$\\
        Critic Learning Rate & $3.193621317706127\times 10^{-5}$\\
        Actor Learning Rate & $3.4187343220724746\times 10^{-5}$\\
        \hline
    \end{tabular}
    \caption{ILBRL Hyperparameters: Halfcheetah}
    \label{tab:ILBRL_Halfcheetah}
\end{table}
\begin{table}[H]
    \centering
    \begin{tabular}{|l|c|}
        \hline
        \textbf{Hyperparameter} & \textbf{Value} \\ \hline
        $\lambda$ & $3.780707811143581$\\
        Critic Learning Rate & $3.193621317706127\times 10^{-5}$\\
        Actor Learning Rate & $3.4187343220724746\times 10^{-5}$\\
        \hline
    \end{tabular}
    \caption{ILBRL Hyperparameters: Walker2d}
    \label{tab:ILBRL_Walker}
\end{table}
\begin{table}[H]
    \centering
    \begin{tabular}{|l|c|}
        \hline
        \textbf{Hyperparameter} & \textbf{Value} \\ \hline
        $\lambda$ & $3.780707811143581$\\
        Critic Learning Rate & $3.193621317706127\times 10^{-5}$\\
        Actor Learning Rate & $3.4187343220724746\times 10^{-5}$\\
        \hline
    \end{tabular}
    \caption{ILBRL Hyperparameters: Ant}
    \label{tab:ILBRL_Ant}
\end{table}
\begin{table}[H]
    \centering
    \begin{tabular}{|l|c|}
        \hline
        \textbf{Hyperparameter} & \textbf{Value} \\ \hline
        Reward Scale & $0.132485274367925$\\
        Discriminator Learning Rate & $0.0004279938571320915$\\
        Actor Learning Rate & $6.308344747206642\times 10^{-5}$\\
        \hline
    \end{tabular}
    \caption{SMODICE Hyperparameters: Hopper}
    \label{tab:SMODICE_hparams_hopper}
\end{table}
\begin{table}[H]
    \centering
    \begin{tabular}{|l|c|}
        \hline
        \textbf{Hyperparameter} & \textbf{Value} \\ \hline
        Reward Scale & $ 0.09806293680143131$\\
        Discriminator Learning Rate & $0.0006509491094594771$\\
        Actor Learning Rate & $8.59713615304861\times 10^{-5}$\\
        \hline
    \end{tabular}
    \caption{SMODICE Hyperparameters: Halfcheetah}
    \label{tab:SMODICE_hparams_halfc}
\end{table}
\begin{table}[H]
    \centering
    \begin{tabular}{|l|c|}
        \hline
        \textbf{Hyperparameter} & \textbf{Value} \\ \hline
        Reward Scale & $ 0.17006151306334566$\\
        Discriminator Learning Rate & $0.00012462906034979525$\\
        Actor Learning Rate & $1.3235001285067025\times 10^{-5}$\\
        \hline
    \end{tabular}
    \caption{SMODICE Hyperparameters: Walker2d}
    \label{tab:SMODICE_hparams_walker}
\end{table}
\begin{table}[H]
    \centering
    \begin{tabular}{|l|c|}
        \hline
        \textbf{Hyperparameter} & \textbf{Value} \\ \hline
        Reward Scale & $ 0.15584724849066423$\\
        Discriminator Learning Rate & $0.0008741470800027616$\\
        Actor Learning Rate & $3.9755828061032105\times 10^{-5}$\\
        \hline
    \end{tabular}
    \caption{SMODICE Hyperparameters: Ant}
    \label{tab:SMODICE_hparams_ant}
\end{table}
\begin{table}[H]
    \centering
    \begin{tabular}{|l|c|}
        \hline
        \textbf{Hyperparameter} & \textbf{Value} \\ \hline
        Reward Scale & $0.1$\\
        Discriminator Learning Rate & $0.0006509491094594771$\\
        Actor Learning Rate & $8.59713615304861\times 10^{-5}$\\
        \hline
    \end{tabular}
    \caption{ORIL Hyperparameters: All Environments}
    \label{tab:ORIL_hparams_all}
\end{table}
\begin{table}[H]
    \centering
    \begin{tabular}{|l|c|}
        \hline
        \textbf{Hyperparameter} & \textbf{Value} \\ \hline
        Learning Rate & $0.01463435996151319$\\
        \hline
    \end{tabular}
    \caption{BC Hyperparameters: Hopper}
    \label{tab:BC_hopper}
\end{table}
\begin{table}[H]
    \centering
    \begin{tabular}{|l|c|}
        \hline
        \textbf{Hyperparameter} & \textbf{Value} \\ \hline
        Learning Rate & $0.017457140313552756$\\
        \hline
    \end{tabular}
    \caption{BC Hyperparameters: Halfcheetah}
    \label{tab:BC_halfcheetah}
\end{table}
\begin{table}[H]
    \centering
    \begin{tabular}{|l|c|}
        \hline
        \textbf{Hyperparameter} & \textbf{Value} \\ \hline
        Learning Rate & $0.01463435996151319$\\
        \hline
    \end{tabular}
    \caption{BC Hyperparameters: Walker}
    \label{tab:BC_walker}
\end{table}
\begin{table}[H]
    \centering
    \begin{tabular}{|l|c|}
        \hline
        \textbf{Hyperparameter} & \textbf{Value} \\ \hline
        Learning Rate & $0.005028663892514822$\\
        \hline
    \end{tabular}
    \caption{BC Hyperparameters: Ant}
    \label{tab:BC_ant}
\end{table}

\end{document}